\pdfoutput=1

\documentclass{article}

\usepackage{times}
\usepackage{graphicx} 
\usepackage{subfigure} 

\usepackage{natbib}
\usepackage{booktabs}
\usepackage{algorithm}
\usepackage{algorithmic}
\usepackage{multirow}

\usepackage[accepted]{icml2017}

\usepackage{amsthm}
\usepackage{amsmath}
\usepackage{amsfonts}
\usepackage{appendix}
\usepackage{tikz}
\usepackage[export]{adjustbox}

\newtheorem{thm}{Theorem}
\newtheorem{cor}{Corollary}
\newtheorem{prop}{Proposition}

\newcommand{\RR}{\mathbb{R}}

\newcommand{\x}{\mathbf{x}}
\newcommand{\bx}{\bar{\mathbf{x}}}
\newcommand{\X}{\mathbf{X}}
\newcommand{\bX}{\bar{\mathbf{X}}}

\newcommand{\y}{\mathbf{y}}
\newcommand{\p}{\mathbf{p}}

\newcommand{\W}{\mathbf{W}}

\newcommand{\II}{\mathbf{1}}
\newcommand{\ee}{\mathbf{e}}
\newcommand{\f}{\mathbf{f}}

\newcommand{\tDL}{\widetilde{\nabla \mathcal{L}}}

\newcommand{\A}{\mathbf{A}}
\newcommand{\B}{\boldsymbol{B}}

\icmltitlerunning{Understanding Synthetic Gradients and DNIs}

\newcommand{\SG}{SG}

\newcommand{\DNI}{DNI}

\def\eg{\emph{e.g.}}
\def\ie{\emph{i.e.}}

\usepackage{bibunits}
\defaultbibliography{example_paper}
\defaultbibliographystyle{icml2017}

\begin{document} 

\twocolumn[
\icmltitle{Understanding Synthetic Gradients and Decoupled Neural Interfaces}



\icmlsetsymbol{equal}{*}

\begin{icmlauthorlist}
\icmlauthor{Wojciech Marian Czarnecki}{dm}
\icmlauthor{Grzegorz \'{S}wirszcz}{dm}
\icmlauthor{Max Jaderberg}{dm}
\icmlauthor{Simon Osindero}{dm}
\icmlauthor{Oriol Vinyals}{dm}
\icmlauthor{Koray Kavukcuoglu}{dm}
\end{icmlauthorlist}

\icmlaffiliation{dm}{DeepMind, London, United Kingdom}

\icmlcorrespondingauthor{WM Czarnecki}{lejlot@google.com}

\icmlkeywords{boring formatting information, machine learning, ICML}

\vskip 0.3in
]



\printAffiliationsAndNotice{}  

\begin{abstract} 


When training neural networks, the use of Synthetic Gradients (\SG) allows layers or modules to be trained without update locking -- without waiting for a true error gradient to be backpropagated -- resulting in Decoupled Neural Interfaces (DNIs).
This unlocked ability of being able to update parts of a neural network asynchronously and with only local information was demonstrated to work empirically in \citet{DNI}. However, there has been very little demonstration of what changes 
\DNI{}s and \SG{}s impose from a functional, representational, and learning dynamics point of view. In this paper, we study DNIs through the use of synthetic gradients on feed-forward networks to better understand their behaviour and elucidate their effect on optimisation. 
We show that the incorporation of \SG{}s does not affect the representational strength of the learning system for a neural network, and prove the convergence of the learning system for linear and deep linear models. On practical problems we investigate the mechanism by which synthetic gradient estimators approximate the true loss, and, surprisingly, how that leads to drastically different layer-wise representations. Finally, we also expose the relationship of using synthetic gradients to other error approximation techniques and find a unifying language for discussion and comparison. 

\end{abstract} 

\vspace{-.3in}
\section{Introduction}

Neural networks can be represented as a graph of computational modules, and training these networks amounts to optimising the weights associated with the modules of this graph to minimise a loss. At present, training is usually performed with first-order gradient descent style algorithms, where the weights are adjusted along the direction of the negative gradient of the loss. In order to compute the gradient of the loss with respect to the weights of a module, one performs backpropagation \cite{williams1986learning} -- sequentially applying the chain rule to compute the exact gradient of the loss with respect to a module. However, this scheme has many potential drawbacks, as well as lacking biological plausibility \cite{marblestone2016toward,bengio2015towards}. In particular, backpropagation results in locking -- the weights of a network module can only be updated after a full forwards propagation of the data through the network, followed by loss evaluation, then finally after waiting for the backpropagation of error gradients. This locking constrains us to updating neural network modules in a sequential, synchronous manner.

\begin{figure}
\centering
\includegraphics[width=0.5\textwidth]{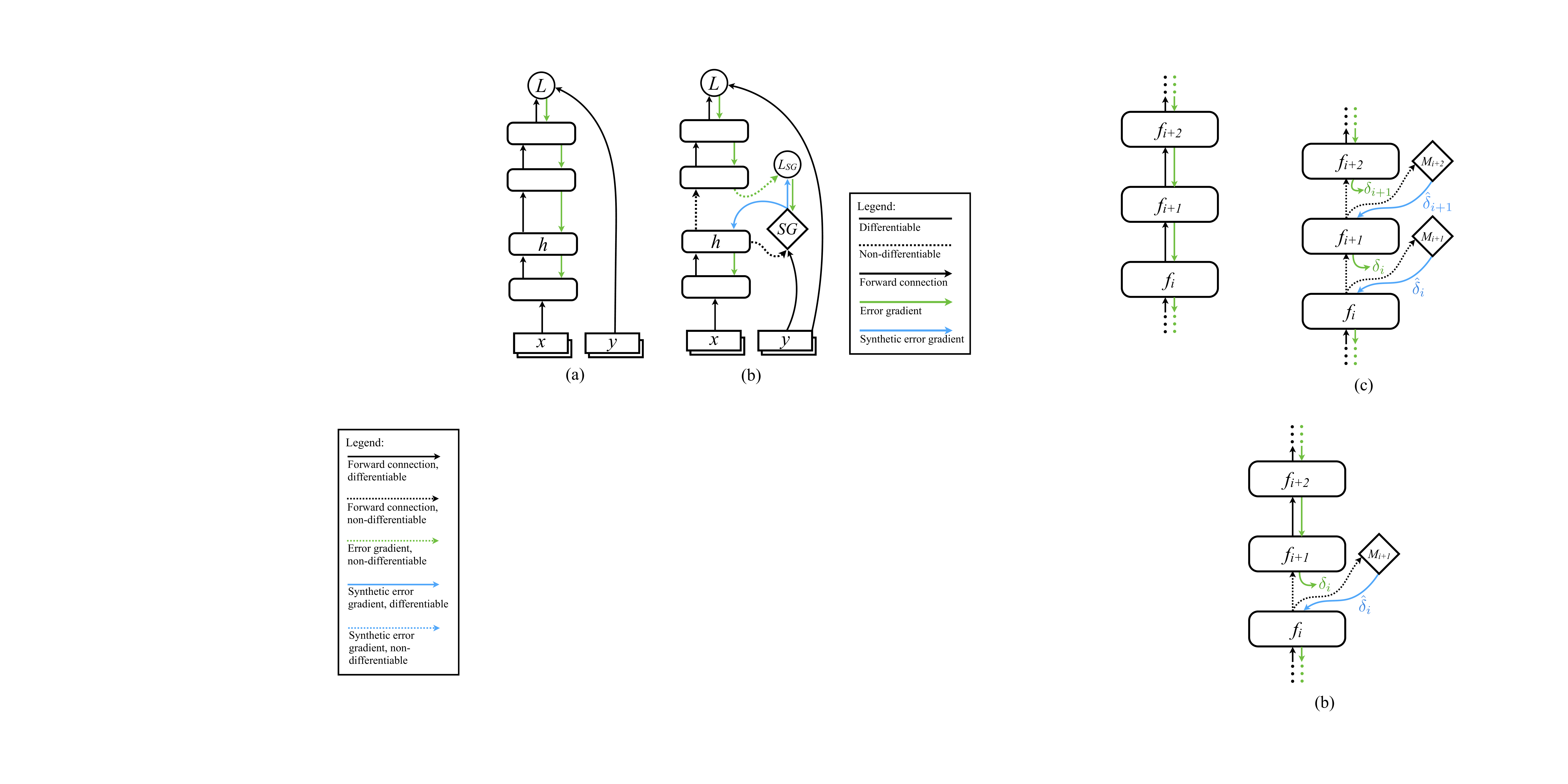}
\caption{Visualisation of \SG{}-based learning (b) vs. regular backpropagation (a).}
\label{fig:dni}
\end{figure}

One way of overcoming this issue is to apply Synthetic Gradients (\SG{}s) to build Decoupled Neural Interfaces (\DNI{}s)~\cite{DNI}. In this approach, models of error gradients are used to approximate the true error gradient. These models of error gradients are local to the network modules they are predicting the error gradient for, so that an update to the module can be computed by using the predicted, synthetic gradients, thus bypassing the need for subsequent forward execution, loss evaluation, and backpropagation. The gradient models themselves are trained at the same time as the modules they are feeding synthetic gradients to are trained. The result is effectively a complex dynamical system composed of multiple sub-networks cooperating to minimise the loss.

There is a very appealing potential of using DNIs \eg~the potential to distribute and parallelise training of networks across multiple GPUs and machines, the ability to asynchronously train multi-network systems, and the ability to extend the temporal modelling capabilities of recurrent networks. However, it is not clear that introducing \DNI{}s and \SG{}s into a learning system will not negatively impact the learning dynamics and solutions found. While the empirical evidence in \citet{DNI} suggests that \SG{}s do not have a negative impact and that this potential is attainable, this paper will dig deeper and analyse the result of using \SG{}s to accurately answer the question of the impact of synthetic gradients on learning systems.

In particular, we address the following questions, using feed-forward networks as our probe network architecture:
\textbf{Does introducing \SG{}s change the critical points of the neural network learning system?} In Section \ref{sec:rep_strength} we show that the critical points of the original optimisation problem are maintained when using \SG{}s.
\textbf{Can we characterise the convergence and learning dynamics for systems that use synthetic gradients in place of true gradients?} Section \ref{sec:learning_dynamics} gives first convergence proofs when using synthetic gradients and empirical expositions of the impact of \SG{}s on learning.
\textbf{What is the difference in the representations and functional decomposition of networks learnt with synthetic gradients compared to backpropagation?} Through experiments on deep neural networks in Section \ref{sec:larger_models}, we find that while functionally the networks perform identically trained with backpropagation or synthetic gradients, the layer-wise functional decomposition is markedly different due to \SG{}s.
%

In addition, in Section \ref{sec:fa} we look at formalising the connection between \SG{}s and other forms of approximate error propagation such as Feedback Alignment~\cite{lillicrap2016random}, Direct Feedback Alignment~\cite{NIPS2016_6441, baldi2016learning}, and Kickback~\cite{balduzzi2014kickback}, and show that all these error approximation schemes can be captured in a unified framework, but crucially only using synthetic gradients can one achieve unlocked training.

\vspace{-.1in}
\section{DNI using Synthetic Gradients} \label{sec:dni_intuitions}
{



}

The key idea of synthetic gradients and DNI is to approximate the true gradient of the loss with a learnt model which predicts gradients without performing full backpropagation.

Consider a feed-forward network consisting of $N$ layers $f_n, n \in \{1,\ldots,N\}$, each taking an input $h^{n-1}_i$ and producing an output $h^{n}_i = f_n(h^{n-1}_i)$, where $h^0_i=x_i$ is the input data point $x_i$. A loss is defined on the output of the network $L_i = L(h^N_i, y_i)$ where $y_i$ is the given label or supervision for $x_i$ (which comes from some unknown $P(y|x)$). Each layer $f_n$ has parameters $\theta_n$ that can be trained jointly to minimise $L_i$ with the gradient-based update rule
\begin{equation*}
\theta_n \leftarrow \theta_n - \alpha~\frac{\partial L(h_i^N,y_i)}{\partial h_i^n} \frac{\partial h_i^n}{\partial \theta_n}
\label{eqn:sgd}
\end{equation*}
where $\alpha$ is the learning rate and $\partial L_i/\partial h_i^n$ is computed with backpropagation.

The reliance on $\partial L_i/\partial h_i^N$ means that an update to layer $i$ can only occur after every subsequent layer $f_j, j \in  \{i+1,\ldots,N\}$ has been computed, the loss $L_i$ has been computed, and the error gradient $\partial L/\partial h_i^N$ backpropgated to get $\partial L_i/\partial h_i^N$. An update rule such as this is \emph{update locked} as it depends on computing $L_i$, and also \emph{backwards locked} as it depends on backpropagation to form $\partial L_i/\partial h_i^n$.

\citet{DNI} introduces a learnt prediction of the error gradient, the \emph{synthetic gradient} $\mathrm{SG}(h_i^n, y_i) = \widehat{\partial L_i/\partial h_i^n} \simeq \partial L_i/\partial h_i^n$ resulting in the update 
\begin{equation*}
\theta_k \leftarrow \theta_k - \alpha~\mathrm{SG}(h_i^n, y_i) \frac{\partial h_i^n}{\partial \theta_k}\;\;\; \forall k \leq n
\label{eqn:dni}
\end{equation*}
%
This approximation to the true loss gradient allows us to have both update and backwards unlocking -- the update to layer $n$ can be applied without any other network computation as soon as $h_i^n$ has been computed, since the \SG{} module is not a function of the rest of the network (unlike $\partial L_i/\partial h_i$). Furthermore, note that since the true $\partial L_i/\partial h_i^n$ can be described completely as a function of just $h_i^n$ and $y_i$, from a mathematical perspective this approximation is sufficiently parameterised.

The synthetic gradient module $\mathrm{SG}(h_i^n, y_i)$ has parameters $\theta_{\mathrm{SG}}$ which must themselves be trained to accurately predict the true gradient by minimising the L$_2$ loss $L_{\mathrm{SG}_i} = \| \mathrm{SG}(h_i^n, y_i) - \partial L_i/\partial h_i^n \|^2$. 

The resulting learning system consists of three decoupled parts: first, the part of the network above the \SG{} module which minimises $L$ wrt. to its parameters $\{\theta_{n+1}, ..., \theta_N\}$, then the \SG{} module that minimises the $L_{\mathrm{SG}}$ wrt. to $\theta_\mathrm{SG}$. Finally the part of the network below the \SG{} module which uses $\mathrm{SG}(h,y)$ as the learning signal to train $\{\theta_1, ... \theta_n\}$, thus it is minimising the loss modeled internally by \SG{}.

\vspace{-.1in}
\subsection*{Assumptions and notation}
\vspace{-.1in}

Throughout the remainder of this paper, we consider the use of a single synthetic gradient module at a single layer $k$ and for a generic data sample $j$ and so refer to $h=h_j=h^k_j$; unless specified we drop the superscript $k$ and subscript $j$. This model is shown in Figure~\ref{fig:dni} (b). We also focus on \SG{} modules which take the point's true label/value as conditioning $\mathrm{SG}(h, y)$ as opposed to $\mathrm{SG}(h)$. Note that without label conditioning, a \SG{} module is trying to approximate not $\partial L / \partial h$ but rather $\mathbb{E}_{P(y|x)} \partial L / \partial h$ since $L$ is a function of both input and label. In theory, the lack of label is a sufficient parametrisation but learning becomes harder, since the \SG{} module has to additionally learn $P(y|x)$.


We also focus most of our attention on models that employ \emph{linear} \SG{} modules, $\mathrm{\SG}(h, y)=hA + yB + C$. Such modules have been shown to work well in practice, and furthermore are more tractable to analyse.

As a shorthand, we denote $\theta_{<h}$ to denote the subset of the parameters contained in modules \emph{up to} $h$ (and symmetrically $\theta_{>h}$), \ie~if $h$ is the $k$th layer then $\theta_{<h} = \{\theta_1 \ldots, \theta_k\}$.

\vspace{-.1in}
\subsection*{Synthetic gradients in operation}
\vspace{-.1in}


Consider an $N$-layer feed-forward network with a single \SG{} module at layer $k$. This network can be decomposed into two sub-networks: the first takes an input $x$ and produces an output $h = F_h(x) = f_k(f_{k-1}(\ldots (f_1(x))))$, while the second network takes $h$ as an input, produces an output $p = F_p(h) = f_N(\ldots(f_{k+1}(h)))$ and incurs a loss $L = L(p, y)$ based on a label $y$. 

With regular backpropagation, the learning signal for the first network $F_h$ is $\partial L / \partial h$, which is a signal that specifies how the input to $F_p$ should be changed in order to reduce the loss. When we attach a linear \SG{} between these two networks, the first sub-network $F_h$ no longer receives the exact learning signal from $F_p$, but an approximation $\mathrm{SG}(h,y)$, which implies that $F_h$ will be minimising an approximation of the loss, because it is using approximate error gradients. Since the \SG{} module is a linear model of $\partial L / \partial h$, the approximation of the true loss that $F_h$ is being optimised for will be a quadratic function of $h$ and $y$. Note that this is \emph{not} what a second order method does when a function is locally approximated with a quadratic and used for optimisation -- here we are approximating the current loss, which is a function of parameters $\theta$ with a quadratic which is a function of $h$. Three appealing properties of an approximation based on $h$ is that $h$ already encapsulates a lot of non-linearities due to the processing of $F_h$, $h$ is usually vastly lower dimensional than $\theta_{<h}$ which makes learning more tractable, and the error only depends on quantities ($h$) which are local to this part of the network rather than $\theta$ which requires knowledge of the entire network.

With the \SG{} module in place, the learning system decomposes into two tasks: the second sub-network $F_p$ tasked with minimising $L$ given inputs $h$, while the first sub-network $F_h$ is tasked with pre-processing $x$ in such a way that the best fitted quadratic approximator of $L$ (wrt. $h$) is minimised. In addition, the \SG{} module is tasked with best approximating $L$.

The approximations and changing of learning objectives (described above) that are imposed by using synthetic gradients may appear to be extremely limiting. However, in both the theoretical and empirical sections of this paper we show that \SG{} models can, and do, learn solutions to highly non-linear problems (such as memorising noise). 

The crucial mechanism that allows such rich behaviour is to remember that the implicit quadratic approximation to the loss implied by the \SG{} module is local (per data point) and non-stationary -- it is continually trained itself. It is not a single quadratic fit to the true loss over the entire optimisation landscape, but a local quadratic approximation specific to each instantaneous moment in optimisation. In addition, because the quadratic approximation is a function only of $h$ and not $\theta$, the loss approximation is still highly non-linear w.r.t. $\theta$.

If, instead of a linear \SG{} module, one uses a more complex function approximator of gradients such as an MLP, the loss is effectively approximated by the integral of the MLP. More formally, the loss implied by the \SG{} module in hypotheses space $\mathcal{H}$ is of class $\{l: \exists g \in \mathcal{H}: \partial l / \partial h = g \}$\footnote{We mean equality for all points where $\partial l / \partial h$ is defined.}.
In particular, this shows an attractive mathematical benefit over predicting loss directly: by modelling gradients rather than losses, we get to implicitly model higher order loss functions.

\vspace{-.1in}
\section{Critical points} \label{sec:rep_strength}


We now consider the effect \SG{} has on critical points of the optimisation problem.
Concretely, it seems natural to ask whether a model augmented with \SG{} is capable of learning the same functions as the original model. We ask this question under the assumption of a locally converging training method, such that we always end up in a critical point. In the case of a \SG{}-based model this implies a set of parameters $\theta$ such that $\partial L / \partial \theta_{> h} = 0$, $\mathrm{\SG}(h, y) \partial h / \partial \theta_{<h}=0$ and $\partial L_{\mathrm{\SG{}}} / \partial \theta_{\mathrm{SG}} = 0$. In other words we are trying to establish whether \SG{} introduces regularisation to the model class, which changes the critical points, or whether it merely introduces a modification to learning dynamics, but retains the same set of critical points.

In general, the answer is positive: \SG{} does induce a regularisation effect. However, in the presence of additional assumptions, we can show families of models and losses for which the original critical points are not affected.

\begin{prop}
Every critical point of the original optimisation problem where \SG{} can produce $\partial L / \partial h_i$ has a corresponding critical point of the \SG-based model. 
\end{prop}
\vspace{-.2in}
\begin{proof}
Directly from the assumption we have that there exists a set of \SG{} parameters such that the loss is minimal, thus $\partial L_{\SG} / \partial \theta_{\mathrm{SG}} = 0$ and also $\mathrm{SG}(h,y) = \partial L / \partial h$ and $\mathrm{\SG}(h, y) \partial h / \partial \theta_{<h}=0$.
\end{proof}
\vspace{-.1in}
The assumptions of this proposition are true for example when $L=0$ (one attains global minimum), when $\partial L / \partial h_i = 0$ or a network is a deep linear model trained with MSE and \SG{} is linear.

In particular, this shows that for a large enough \SG{} module all the critical points of the original problem have a corresponding critical point in the \SG-based model.
Limiting the space of \SG{} hypotheses leads to inevitable reduction of number of original critical points, thus acting as a regulariser. 
At first this might look like a somewhat negative result, since in practice we rarely use a \SG{} module capable of exactly producing true gradients.
However, there are three important observations to make: (1) Our previous observation reflects having an exact representation of the gradient at the critical point, not in the whole parameter space. (2) One does preserve all the critical points where the loss is zero, and given current neural network training paradigms these critical points are important. For such cases even if \SG{} is linear the critical points are preserved. (3) In practice one rarely optimises to absolute convergence regardless of the approach taken; rather we obtain numerical convergence meaning that $\|\partial L / \partial \theta\|$ is small enough. Thus, all one needs from \SG-based model is to have small enough $\|(\partial L / \partial h + e) \partial h / \partial \theta_{< h}\| \leq \|\partial L / \partial \theta_{< h} \| + \|e\| \| \partial h / \partial \theta_{< h} \|$, implying that the approximation error at a critical point just has to be small wrt to $\|\partial h / \partial \theta_{< h}\|$ and need not be 0.

To recap: so far we have shown that \SG{} can preserve critical points of the optimisation problem
. However, \SG{} can also introduce \emph{new} critical points, leading to premature convergence and spurious additional solutions. As with our previous observation, this does not effect \SG{} modules which are able to represent gradients exactly. But if the \SG{} hypothesis space does not include a good approximator\footnote{In this case, our gradient approximation needs to be reasonable at every point through optimisation, not just the critical ones.} of the true gradient, then we can get new critical points which end up being an equilibrium state between \SG{} modules and the original network. We provide an example of such an equilibrium in the Supplementary Materials Section~\ref{SM:examples}. 

\vspace{-.1in}
\section{Learning dynamics} \label{sec:learning_dynamics}

Having demonstrated that important critical points are preserved and also that new ones might get created, we need a better characterisation of the basins of attraction, and to understand when, in both theory and practice, one can expect convergence to a good solution.
%
\vspace{-.1in}
\subsection*{Artificial Data}
\vspace{-.1in}

We conduct an empirical analysis of the learning dynamics on easily analysable artificial data. We create 2 and 100 dimensional versions of four basic datasets (details in the Supplementary Materials Section~\ref{SM:experiments}) and train four simple models (a linear model and a deep linear one with 10 hidden layers, trained to minimise MSE and log loss) with regular backprop and with a \SG{}-based alternative to see whether it (numerically) converges to the same solution.

For MSE and both shallow and deep linear architectures the \SG-based model converges to the global optimum  (exact numerical results provided in Supplementary Material Table~\ref{tab:diffs}). However, this is not the case for logistic regression. This effect is a direct consequence of a linear \SG{} module being unable to model $\partial L / \partial p$\footnote{$\partial L / \partial p = \exp(xW+b) / (1 + \exp(xW+b)) - y$} (where $p=xW+b$ is the output of logistic regression), which often approaches the step function (when data is linearly separable), and cannot be well approximated with a linear function $\mathrm{\SG}(h, y)=hA + yB + C$. 
Once one moves towards problems without this characteristic (\eg~random labeling) the problem vanishes, since now $\partial L / \partial p$ can be approximated much better. While this may not seem particularly significant, it illustrates an important characteristic of \SG{} in the context of the log loss -- it will struggle to overfit to training data, since it requires modeling step function type shapes, which is not possible with a linear model. In particular this means that for best performance one should adapt the \SG{} module architecture to the loss function used \textemdash for MSE linear \SG{} is a reasonable choice, however for log loss one should use architectures including a sigmoid $\sigma$ applied pointwise to a linear \SG{}, such as $\mathrm{SG}(h,y) = d\sigma(hA) + yB + C$.

\begin{figure*}[h]
\begin{tikzpicture}
    \node[anchor=north west] at (0,3.5) {
        \includegraphics[height=3.1cm]{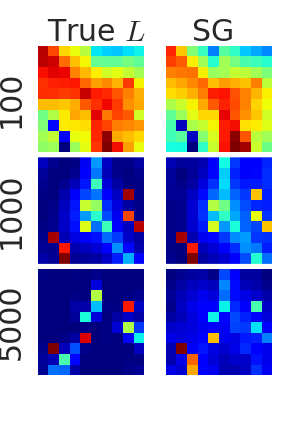}
        \includegraphics[height=3.1cm]{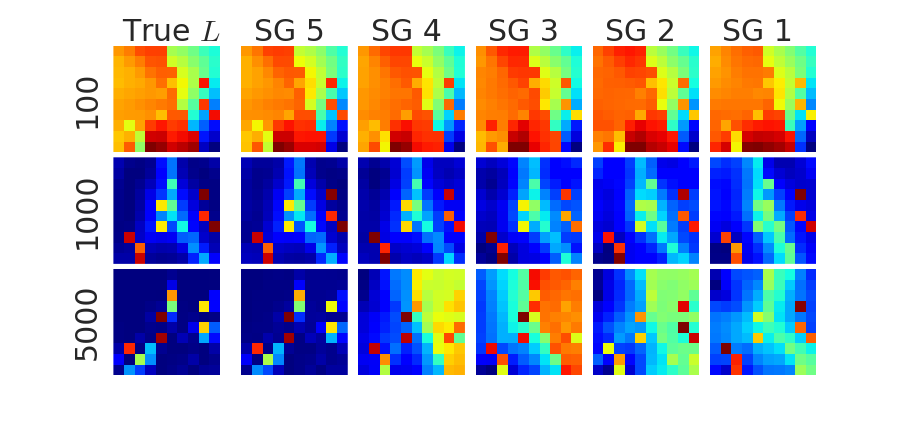}
        \includegraphics[height=3.1cm]{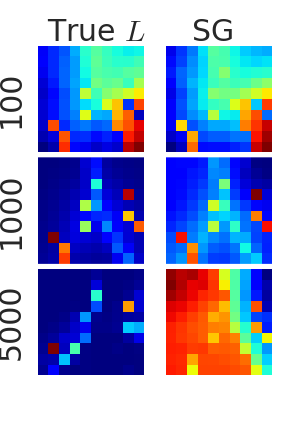}
        \includegraphics[height=3.1cm]{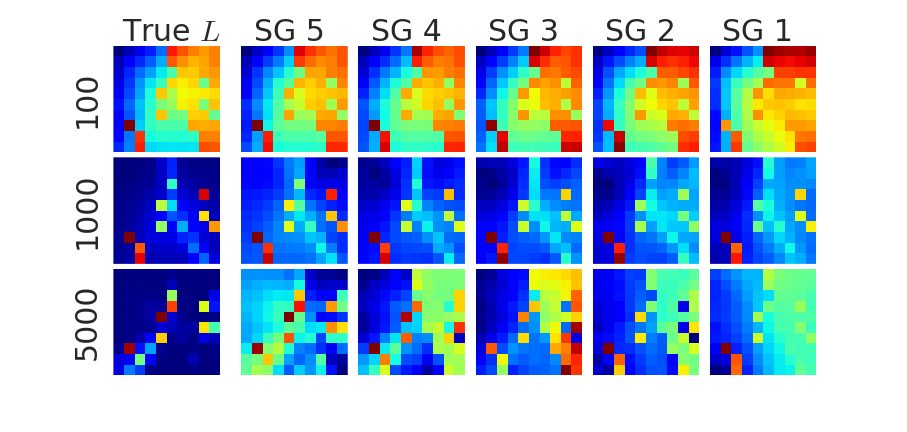}
    };
    \node[anchor=north west] at (0,0) {
        \includegraphics[height=3.1cm]{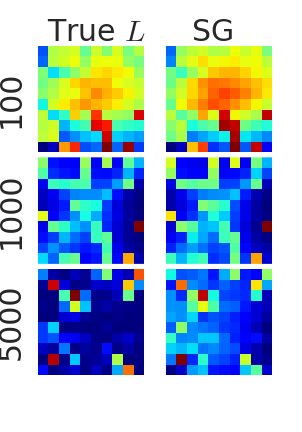}
        \includegraphics[height=3.1cm]{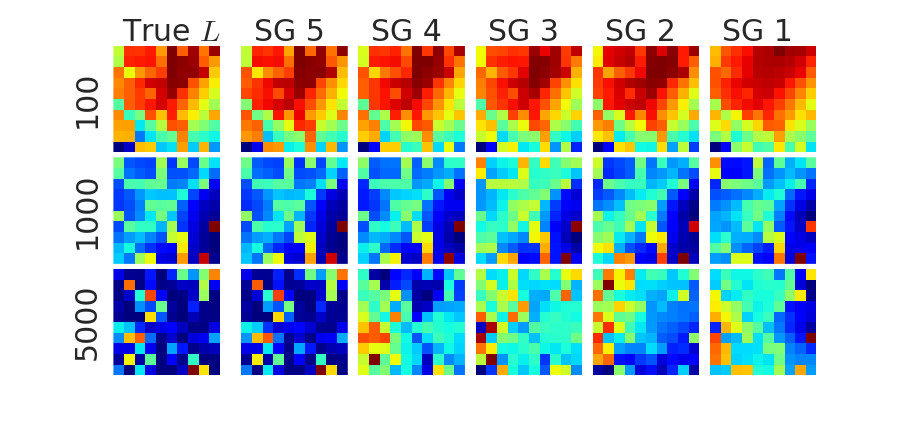}
        \includegraphics[height=3.1cm]{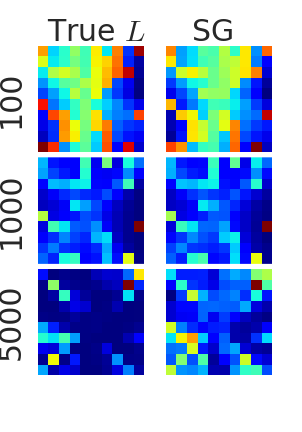}
        \includegraphics[height=3.1cm]{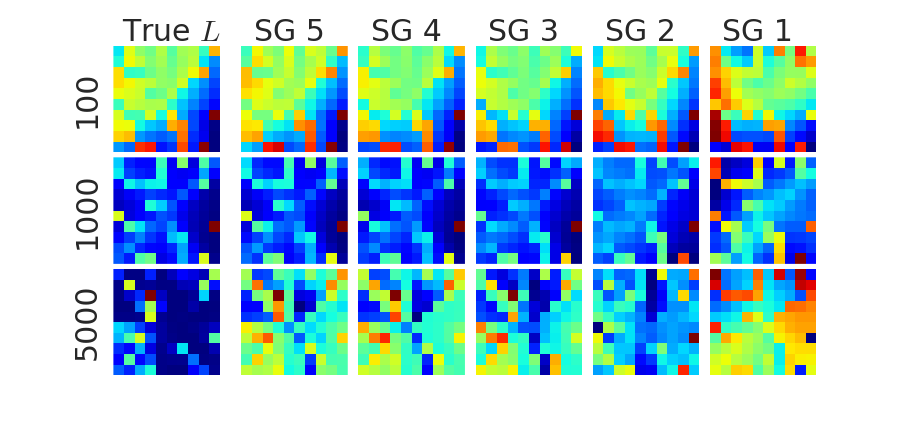}
    };
    \draw (8.85, -3.5) -- (8.85, 3.5);
    \draw (0, 0) -- (17, 0);
    
    \node[anchor=west, align=left, rotate=90] at (-0.05, 0.75) {Train iteration};
    
    \node[anchor=west, align=left] at (0.35,3.5) {Single SG};
    \node[anchor=west, align=left] at (4.35,3.5) {Every layer SG};
    
    \node[anchor=west, align=left] at (0.05,0.25) {a) MSE, noisy linear data};

    \node[anchor=west, align=left] at (8.9 + 0.3,3.5) {Single SG};
    \node[anchor=west, align=left] at (8.9 + 4.3,3.5) {Every layer SG};

    \node[anchor=west, align=left] at (8.9,0.25) {b) log loss, noisy linear data};

    \node[anchor=west, align=left, rotate=90] at (-0.05, -2.75) {Train iteration};
    \node[anchor=west, align=left] at (0.05,-3.25) {c) MSE, randomly labeled data};

    \node[anchor=west, align=left] at (8.9,-3.25) {d) log loss, randomly labeled data};
\end{tikzpicture}
\caption{Visualisation of the true MSE loss and the loss approximation reconstructed from \SG{} modules, when learning points are arranged in a 2D grid, with linearly separable 90\% of points and 10\% with randomly assigned labels (top row) and with completely random labels (bottom row). The model is a 6 layers deep relu network. Each image consists of visualisations for a model with a single \SG{} (left part) and with \SG{} between every two layers (on the right). Note, that each image has an independently scaled color range, since we are only interested in the shape of the surface, not particular values (which cannot be reconstructed from the \SG). Linear \SG{} tracks the loss well for MSE loss, while it struggles to fit to log loss towards the end of the training of nearly separable data. Furthermore, the quality of loss estimation degrades towards the bottom of the network when multiple \SG{}s bootstrap from each other.}
\label{fig:lossrec}
\end{figure*}
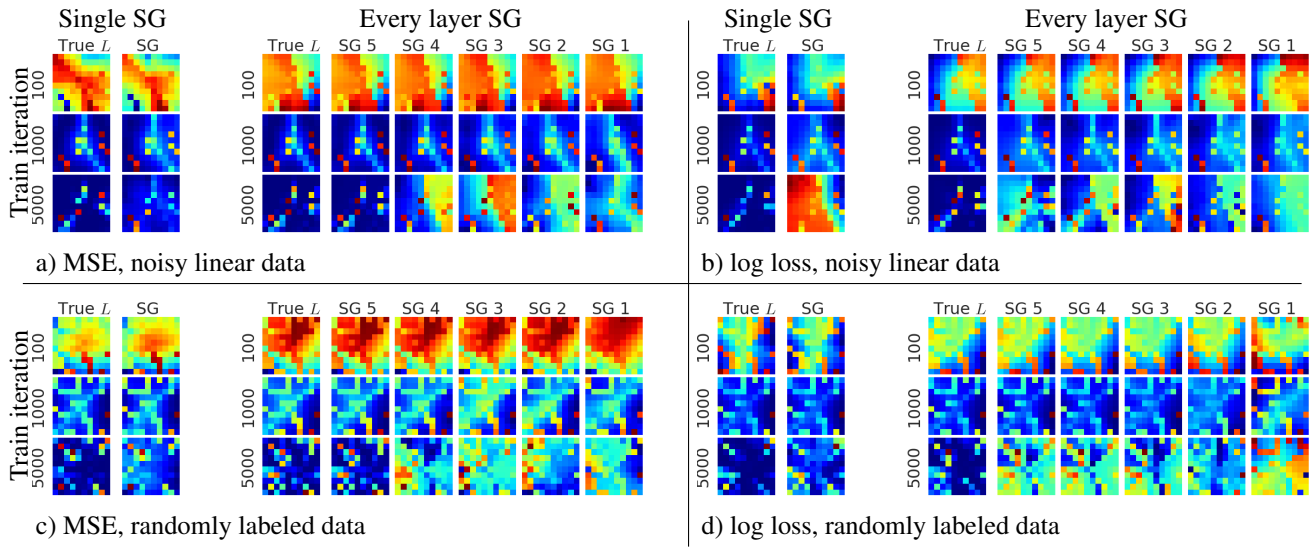

As described in Section 2, using a linear \SG{} module makes the implicit assumption that loss is a quadratic function of the activations. Furthermore, in such setting we can actually reconstruct the loss being used up to some additive constant since $\partial L / \partial h = hA + yB + C$ implies that $L(h) = \tfrac{1}{2} hAh^T + (yB + C)h^T + const$. If we now construct a 2-dimensional dataset, where data points are arranged in a 2D grid, we can visualise the loss implicitly predicted by the \SG{} module and compare it with the true loss for each point. 

Figure~\ref{fig:lossrec} shows the results of such an experiment when learning a highly non-linear model (5-hidden layer relu network). As one can see, the quality of the loss approximation has two main components to its dynamics. First, it is better in layers closer to the true loss (\ie~the topmost layers), which matches observations from \citet{DNI} and the intuition that the lower layers solve a more complex problem (since they bootstrap their targets). Second, the loss is better approximated at the very beginning of the training and the quality of the approximation degrades slowly towards the end. This is a consequence of the fact that close to the end of training, the highly non-linear model has quite complex derivatives which cannot be well represented in a space of linear functions. It is worth noting, that in these experiments, the quality of the loss approximation dropped significantly when the true loss was around 0.001, thus it created good approximations for the majority of the learning process. There is also an empirical confirmation of the previous claim, that with log loss and data that can be separated, linear \SG{}s will have problems modeling this relation close to the end of training (Figure~\ref{fig:lossrec} (b) left), while there is no such problem for MSE loss (Figure~\ref{fig:lossrec} (a) left).

\vspace{-.1in}
\subsection*{Convergence}
\vspace{-.1in}

It is trivial to note that if a \SG{} module used is globally convergent to the true gradient, and we only use its output after it converges, then the whole model behaves like the one trained with regular backprop. However, in practice we never do this, and instead train the two models in parallel without waiting for convergence of the \SG{} module. We now discuss some of the consequences of this, and begin by showing that as long as a synthetic gradient produced is close enough to the true one we still get convergence to the true critical points. Namely, only if the error introduced by \SG{}, backpropagated to all the parameters, is consistently smaller than the norm of true gradient multiplied by some positive constant smaller than one, the whole system converges. Thus, we essentially need the \SG{} error to vanish around critical points.

\begin{prop}
Let us assume that a \SG{} module is trained in each iteration in such a way that it $\epsilon$-tracks true gradient, i.e.  that $\|\mathrm{\SG}(h, y) - \partial L/\partial h\| \leq \epsilon$. If $\|\partial h / \partial \theta_{<h}\|$ is upper bounded by some $K$ and there exists a constant $\delta \in (0,1)$ such that in every iteration $\epsilon K \leq \|\partial L/\partial \theta_{<h}\| \tfrac{1-\delta}{1+\delta}$, then the whole training process converges to the solution of the original problem.
\label{prop:epsilon}
\end{prop}
\vspace{-.2in}
\begin{proof}
Proof follows from showing that, under the assumptions, effectively we are training with noisy gradients, where the noise is small enough for convergence guarantees given by \citet{zoutendijk1970nonlinear, gratton2011much} to apply. Details are provided in the Supplementary Materials Section~\ref{SM:proofs}.
\end{proof}
\vspace{-.1in}
As a consequence of  Proposition \ref{prop:epsilon} we can show that with specifically chosen learning rates (not merely ones that are small enough) we obtain convergence for deep linear models.

\begin{cor}
For a deep linear model minimising MSE, trained with a linear \SG{} module attached between two of its hidden layers, there exist learning rates in each iteration such that it converges to the critical point of the original problem.
\end{cor}
\vspace{-.2in}
\begin{proof}
Proof follows directly from Propositions 1 and 2.
Full proof is given in Supplementary Materials Section~\ref{SM:proofs}.
\end{proof}
\vspace{-.1in}
\begin{figure*}[hbt]
\includegraphics[width=0.45\textwidth]{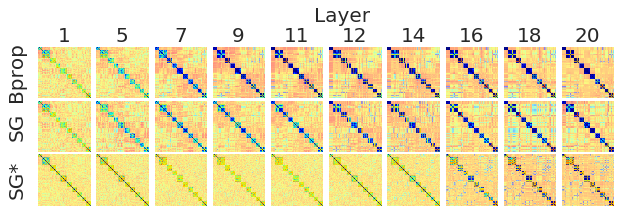}
\includegraphics[width=0.55\textwidth]{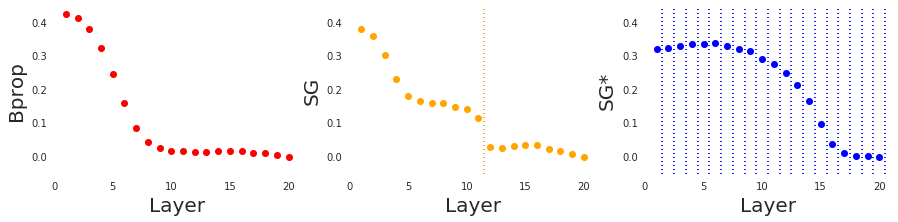}
\caption{(left) Representation Dissimilarity Matrices for a label ordered sample from MNIST dataset pushed through 20-hidden layers deep relu networks trained with backpropagation (top row), a single \SG{} attached between layers 11 and 12 (middle row) and \SG{} between every pair of layers (bottom row). Notice the appearance of dark blue squares on a diagonal in each learning method, which shows when a clear inner-class representation has been learned. For visual confidence off block diagonal elements are semi transparent. (right) L$_2$ distance between diagonal elements at a given layer and the same elements at layer 20. Dotted lines show where \SG{}s are inserted.}
\label{fig:rdms}
\end{figure*}
For a shallow model we can guarantee convergence to the global solution provided we have a small enough learning rate, which is the main theoretical result of this paper.


\begin{thm}
Let us consider linear regression trained with a linear \SG{} module attached between its output and the loss. If one chooses the learning rate of the \SG{} module using line search, then in every iteration there exists small enough, positive learning rate of the main network such that it converges to the global solution.
\end{thm}
\vspace{-.2in}
\begin{proof}
The general idea (full proof in the Supplementary Materials Section~\ref{SM:proofs}) is to show that with assumed learning rates the sum of norms of network error and \SG{} error decreases in every iteration.
\end{proof}
\vspace{-.1in}
Despite covering a quite limited class of models, these are the very first convergence results for \SG{}-based learning. Unfortunately, they do not seem to easily generalise to the non-linear cases, which we leave for future research. 



\vspace{-.1in}
\section{Trained models} \label{sec:larger_models}


We now shift our attention to more realistic data. We train deep relu networks of varied depth (up to 50 hidden layers) with batch-normalisation and with two different activation functions on MNIST and compare models trained with full backpropagation to variants that employ a \SG{} module in the middle of the hidden stack.

\begin{figure}[h]
\includegraphics[width=0.5\textwidth]{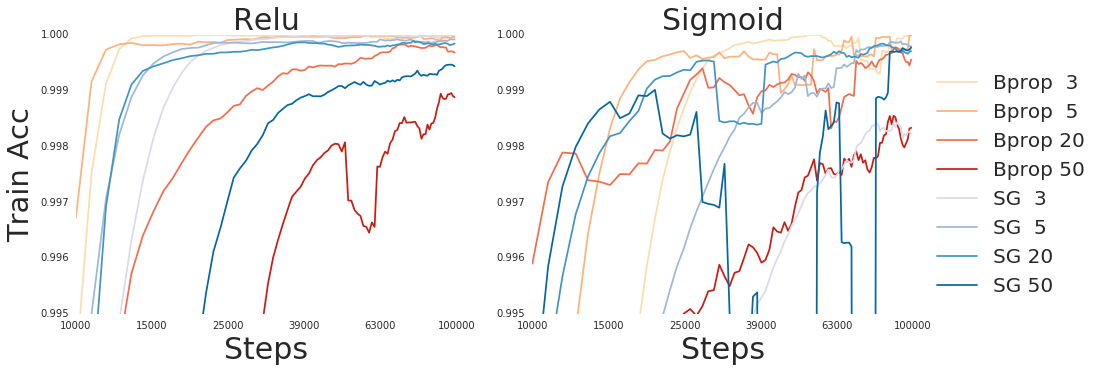}
\caption{Learning curves for MNIST experiments with backpropagation and with a single \SG{} in a stack of from 3 to 50 hidden layers using one of two activation functions: relu and sigmoid.}
\label{fig:mnist}
\end{figure}

Figure~\ref{fig:mnist} shows, that \SG{}-based architectures converge well even if there are many hidden layers both below and above the module.  Interestingly, \SG{}-based models actually seem to converge faster (compare for example 20- or 50 layer deep relu network). We believe this may be due to some amount of loss function smoothing since, as described in Section \ref{sec:dni_intuitions}, a linear \SG{} module effectively models the loss function to be quadratic -- thus the lower network has a simpler optimisation task and makes faster learning progress. 

Obtaining similar errors on MNIST does not necessarily mean that trained models are the same or even similar. Since the use of synthetic gradients can alter learning dynamics and introduce new critical points, they might converge to different types of models. Assessing the representational similarity between different models is difficult, however.
%
%
%
One approach is to compute and visualise Representational Dissimilarity Matrices~\cite{kriegeskorte2008representational} for our data. We sample a subset of 400 points from MNIST, order them by label, and then record activations on each hidden layer when the network is presented with these points. 
We plot the pairwise correlation matrix for each layer, as shown in Figure \ref{fig:rdms}. This representation is permutation invariant, and thus the emergence of a block-diagonal correlation matrix means that at a given layer, points from the same class already have very correlated representations.

Under such visualisations one can notice qualitative differences between the representations developed under standard backpropagation training versus those delivered by a \SG{}-based model. 
In particular, in the MNIST model with 20 hidden layers trained with standard backpropagation we see that the representation covariance after 9 layers is nearly the same as the final layer's representation. However, by contrast, if we consider the same architecture but with a \SG{} module in the middle we see that the layers before the \SG{} module develop a qualitatively different style of representation. Note: this does \emph{not} mean that layers before \SG{} do not learn anything useful. To confirm this, we also introduced linear classifier probes \cite{alain2016understanding} and observed that, as with the pure backpropagation trained model, such probes can achieve 100\% training accuracy after the first two hidden-layers of the \SG{}-based model, as shown in Supplementary Material's Figure~\ref{fig:probes}.
%
With 20 \SG{} modules (one between every pair of layers), the representation is scattered even more broadly:  we see rather different learning dynamics, with each layer contributing a small amount to the final solution, and there is no longer a point in the progression of layers where the representation is more or less static in terms of correlation structure (see Figure~\ref{fig:rdms}).

Another way to investigate whether the trained models are qualitatively similar is to examine the norms of the weight matrices connecting consecutive hidden layers, and to assess whether the general shape of such norms are similar. While this does not definitively say anything about how much of the original classification is being solved in each hidden layer, it is a reasonable surrogate for how much computation is being performed in each layer\footnote{We train with a small L$_2$ penalty added to weights to make norm correspond roughly to amount of computation.}.
\begin{figure}[h]
\includegraphics[width=0.47\textwidth]{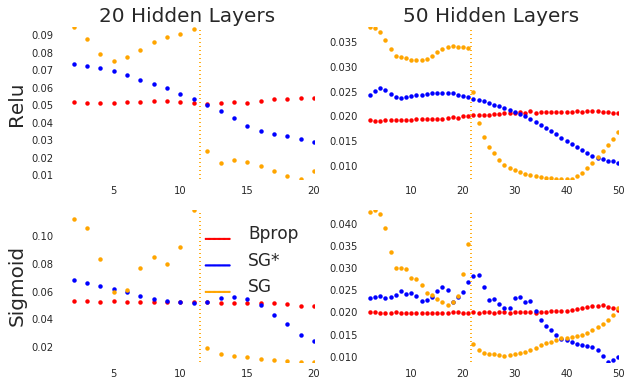}
\caption{Visualisation of normalised squared norms of linear transformations in each hidden layer of every model considered. Dotted orange line denotes level at which a single \SG{} is attached. \SG{}* has a \SG{} at every layer.}
\label{fig:mnistnorms}
\end{figure}
According to our experiments (see Figure~\ref{fig:mnistnorms} for visualisation of one of the runs), models trained with backpropagation on MNIST tend to have norms slowly increasing towards the output of the network (with some fluctuations and differences coming from activation functions, random initialisations, etc.). If we now put a \SG{} in between every two hidden layers, we get norms that start high, and then decrease towards the output of the network (with much more variance now). Finally, if we have a single \SG{} module we can observe that the behaviour after the \SG{} module resembles, at least to some degree, the distributions of norms obtained with backpropagation, while before the \SG{} it is more chaotic, with some similarities to the distribution of weights with \SG{}s in-between every two layers.

These observations match the results of the previous experiment and the qualitative differences observed. When synthetic gradients are used to deliver full unlocking we obtain a very basic model at the lowest layers and then see iterative corrections in deeper layers. For a one-point unlocked model with a single \SG{} module, we have two slightly separated models where one behaves similarly to backprop, and the other supports it. Finally, a fully locked model (\ie~traditional backprop) solves the task relatively early on, and later just increases its confidence.

We note that the results of this section support our previous notion that we are effectively dealing with a multi-agent system, which looks for coordination/equilibrium between components, rather than a single model which simply has some small noise injected into the gradients (and this is especially true for more complex models). 

\vspace{-.1in}
\section{\SG{} and conspiring networks}
\label{sec:fa}

\begin{table*}[htb]

\begin{tabular}{p{1.75cm}p{14.75cm}}
\toprule
Network & \multirow{2}{*}{ \includegraphics[width=0.875\textwidth]{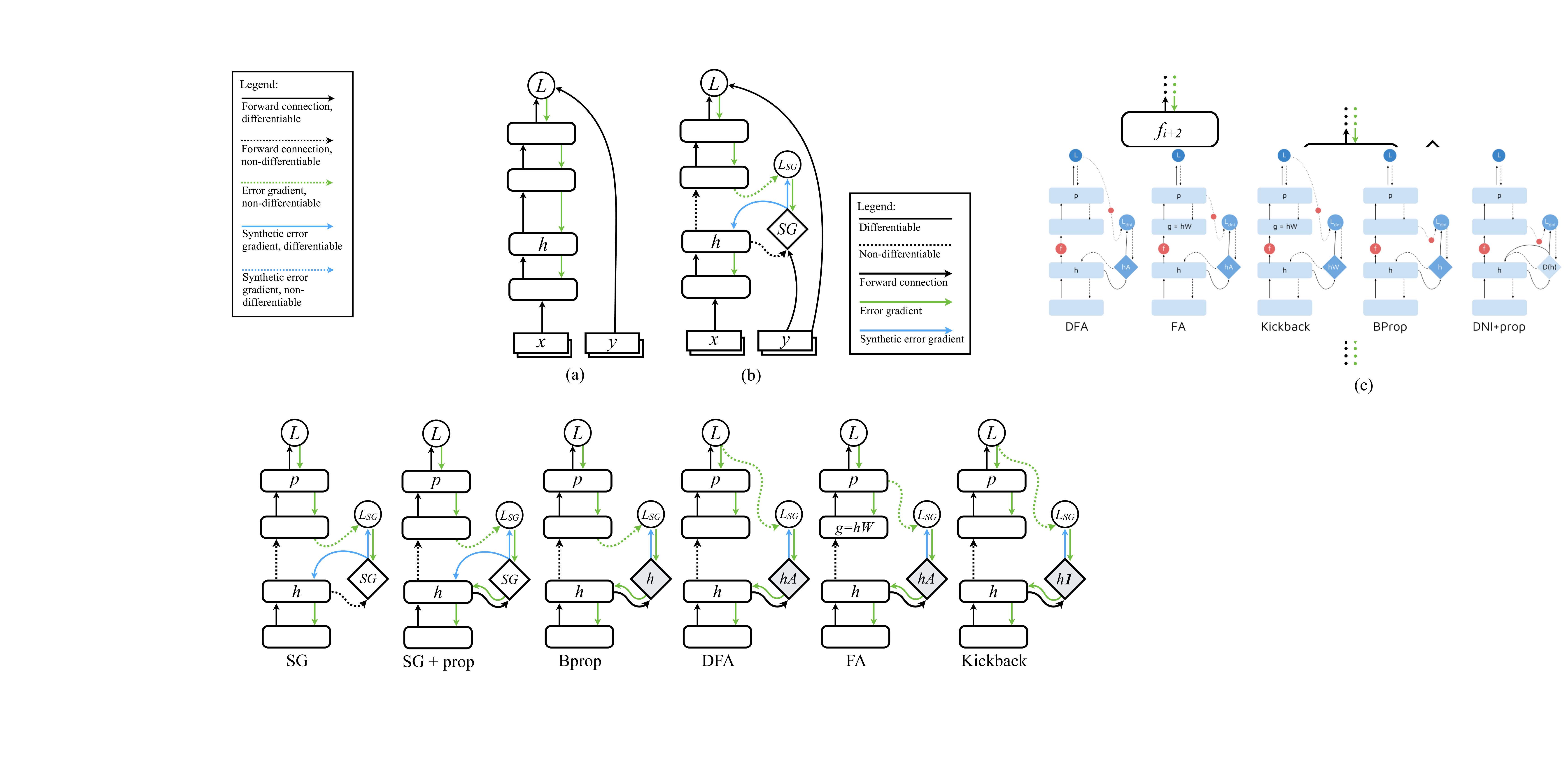} }\\[24ex]
Method
\end{tabular}

\begin{tabular}{lp{1.54cm}p{2.7cm}p{2cm}p{2cm}p{2cm}p{2cm}}
\midrule
$\widehat{\partial L / \partial h}$ & $\mathrm{SG}(h,y)$ & $\mathrm{SG}(h,y) + \alpha \tfrac{\partial L_{\mathrm{SG}}}{ \partial h} $ & $\partial L / \partial h$ & $(\partial L / \partial p)A^T$ & $(\partial L / \partial g)A^T$ & $(\partial L / \partial p)\boldsymbol{1}^T$\\
$\mathrm{SG}(h,y)$ & $\mathrm{SG}(h,y)$ & $\mathrm{SG}(h,y)$ & $h$ & $hA$ & $hA$ & $h\boldsymbol{1}$ \\
\SG{} trains& yes & yes & no & no & no & no \\
\SG{} target & $\partial L/\partial h$ & $\partial L/\partial h$ & $-\partial L/\partial h$ & $-\partial L/\partial p$ & $-\partial L/\partial g$ & $-\partial L/\partial p$ \\
$L_\mathrm{SG}(t, s)$ & $\|t-s\|^2$ & $\|t-s\|^2$ & $ - \langle t, s\rangle $  & $ - \langle t, s\rangle $  & $ - \langle t, s\rangle $ & $ - \langle t, s\rangle $ \\
\midrule
Update locked & no & yes* & yes & yes & yes & yes \\
Backw. locked & no & yes* & yes & no & yes & no \\
Direct error & no & no & no &yes & no & yes\\
\bottomrule
\end{tabular}
\caption{Unified view of ``conspiring'' gradients methods, including backpropagation, synthetic gradients are other error propagating methods. For each of them, one still trains with regular backpropagation (chain rule) however $\partial L / \partial h$ is substituted with a particular $\widehat{\partial L / \partial h}$. Black lines are forward signals, blue ones are synthetic gradients, and green ones are true gradients. Dotted lines represent non-differentiable operations. The grey modules are not trainable. $A$ is a fixed, random matrix and $\boldsymbol{1}$ is a matrix of ones of an appropriate dimension. * In SG+Prop the network is locked if there is a single \SG{} module, however if we have multiple ones, then propagating error signal only locks a module with the next one, not with the entire network. Direct error means that a model tries to solve classification problem directly at layer $h$.}
\label{tab:dfa}
\end{table*}


We now shift our attention and consider a unified view of several different learning principles that work by replacing true gradients with surrogates. We focus on three such approaches: Feedback Alignment (FA) ~\cite{lillicrap2016random}, Direct Feedback Alignment (DFA) ~\cite{NIPS2016_6441}, and Kickback (KB)~\cite{balduzzi2014kickback}. FA effectively uses a fixed random matrix during backpropagation, rather than the transpose of the weight matrix used in the forward pass. DFA does the same, except each layer directly uses the learning signal from the output layer rather than the subsequent local one. KB also pushes the output learning signal directly but through a predefined matrix instead of a random one.
%
%
By making appropriate choices for targets, losses, and model structure we can cast all of these methods in the \SG{} framework, and view them as comprising two networks with a \SG{} module in between them, wherein the first module builds a representation which makes the task of the \SG{} predictions easier.

We begin by noting that in the \SG{} models described thus far we do not backpropagate the \SG{} error back into the part of the main network preceding the \SG{} module (\ie~we assume $\partial L_\mathrm{SG} / \partial h = 0$).  However, if we relax this restriction, we can use this signal (perhaps with some scaling factor $\alpha$) and obtain what we will refer to as a $\SG{} + \mathrm{prop}$ model. 
Intuitively, this additional learning signal adds capacity to our \SG{} model and forces both the main network and the $\SG{}$ module to ``conspire'' towards a common goal of making better gradient predictions.
%
From a practical perspective, according to our experiments, this additional signal heavily stabilises learning system%
\footnote{ 
In fact, ignoring the gradients \emph{predicted} by \SG{} and only using the  derivative of the \SG{} loss, i.e. $\partial L_\mathrm{SG} / \partial h$, still provides enough learning signal to converge to a solution for the original task in the simple classification problems we considered.
%
%
We posit a simple rationale for this: if one can predict gradients well using a simple transformation of network activations (\eg~a linear mapping), this suggests that the loss itself can be predicted well too, and thus (implicitly) so can the correct outputs. 
%
}.
However, this comes at the cost of no longer being unlocked.

Our main observation in this section is that FA, DFA, and KB can be expressed in the language of ``conspiring'' networks (see Table~\ref{tab:dfa}), of two-network systems that use a \SG{} module.
The only difference between these approaches is how one parametrises \SG{} and what target we attempt to fit it to. 
This comes directly from the construction of these systems, and the fact that if we treat our targets as constants (as we do in \SG{} methods), then the backpropagated error from each \SG{} module ($\partial L_\mathrm{SG}/\partial h$) matches the prescribed update rule of each of these methods ($\widehat{\partial L / \partial h}$).
One direct result from this perspective is the fact that Kickback is essentially DFA with $A=\boldsymbol{1}$. 
%
%
%
For completeness, we note that regular backpropagation can also be expressed in this unified view --  to do so, we construct a \SG{} module such that the gradients it produces attempt to align the layer activations with the negation of the true learning signal ($-\partial L / \partial h$). 
%
%
%
In addition to unifying several different approaches, our mapping also illustrates the potential utility and diversity in the generic idea of predicting gradients.

\vspace{-.1in}
\section{Conclusions}

This paper has presented new theory and analysis for the behaviour of synthetic gradients in feed forward models.
Firstly, we showed that introducing \SG{} does not necessarily change the critical points of the original problem, however at the same time it can introduce new critical points into the learning process. This is an important result showing that \SG{} does not act like a typical regulariser despite simplifying the error signals.
%
%
Secondly, we showed that (despite modifying learning dynamics) \SG{}-based models converge to analogous solutions to the true model under some additional assumptions. We proved exact convergence for a simple class of models, and for more complex situations we demonstrated that the implicit loss model captures the characteristics of the true loss surface. It remains an open question how to characterise the learning dynamics in more general cases.
%
%
%
Thirdly, we showed that despite these convergence properties the trained networks can be qualitatively different from the ones trained with backpropagation.
While not necessarily a drawback, this is an important consequence one should be aware of when using synthetic gradients in practice.
%
%
Finally, we provided a unified framework that can be used to describe alternative learning methods such as Synthetic Gradients, FA, DFA, and Kickback, as well as standard Backprop. The approach taken shows that the language of \emph{predicting gradients}
is suprisingly universal and provides additional intuitions and insights into the models.

\section*{Acknowledgments}
The authors would like to thank James Martens and Ross Goroshin for 
their valuable remarks and discussions.

\bibliographystyle{icml2017}
\bibliography{example_paper}

\begin{appendices}
\section*{Supplementary Materials}

\section{Additional examples}
\label{SM:examples}

\subsection*{Critical points}

We can show an example of \SG{} introducing new critical points. Consider a small one-dimensional training dataset $\{-2, -1, 1, 2\} \subset \mathbb{R}$, and let us consider a simple system where the model $f : \mathbb{R} \rightarrow \mathbb{R}$ is parametrised with two scalars, $a$ and $b$ and produces $ax+b$. 
We train it to minimise $L(a,b) = \sum_{i=1}^4 |ax_i+b|$.
This has a unique minimum which is obtained for $a=b=0$, and standard gradient based methods will converge to this solution. 
Let us now attach a \SG{} module between$f$ and $L$.
This module produces a (trainable) scalar value $c \in \mathbb{R}$ (thus it produces a single number, independent from the input). 
%
%
Regardless of the value of $a$, we have a critical point of the \SG{} module when $b=0$ and $c=0$. However, solutions with $a=1$ and $c=0$ are clearly not critical points of the original system.
%
Figure~\ref{fig:rep} shows the loss surface and the fitting of \SG{} module when it introduces new critical point.
\begin{figure}[htb]
\includegraphics[width=.23\textwidth]{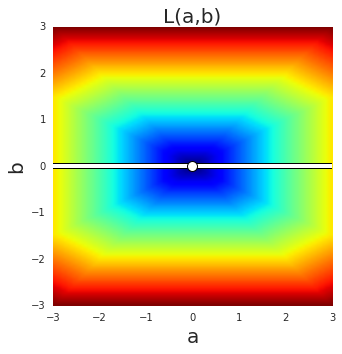}
\includegraphics[width=.23\textwidth]{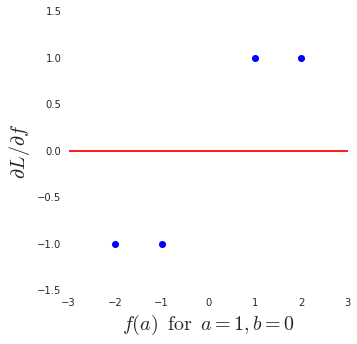}
\caption{\emph{Left:} The loss surface with a white marker represents critical point of the original optimisation and white line a set of critical points of \SG{} based one. \emph{Right:} A situation when \SG{} finds a solution $d=0$ which introduces new critical point, which is not a critical point of the original problem.}
\label{fig:rep}
\end{figure}

\section{Proofs}
\label{SM:proofs}

\textbf{Theorem 1}\emph{
Let us consider linear regression trained with a linear \SG{} module attached between its output and the loss. If one chooses the learning rate of the \SG{} module using line search, then in every iteration there exists small enough, positive learning rate of the main network such that it converges to the global solution.
}

\newcommand{\ekp}[1]{\ee_{#1}^{\parallel}}
\newcommand{\BB}{\mathbf{B}}

\begin{proof}
Let $\X = \{\x^s\}_{s=1}^{S} \in \RR^{d \times S}$ be the data, let $\{y_s\}_{s=1}^{S} \in \RR^{1\times S}$ be the labels. Throughout the proof $k$ will be the iteration of training.\\
%
We denote by $\II \in \RR^{1\times S}$ a row vector in which every element is $1$.
We also follow the standard convention of including the bias in the weight matrix by augmenting the data $\X$ with one extra coordinate always equal to $1$. Thus, we denote $\bX = (\X^T| \II^T)^T$, $\bX \in \RR^{(d+1)\times S}$ and $\bx^s$-the columns of $\bX$.
Using that convention, the weight matrix is $\W_k \in \RR^{1 \times (d+1)}$. We have
\[
p^s_k := \W_k \bx^s ,
\]
\[
L = \frac{1}{2} \sum\limits_{s=1}^{S} \left(y^s - p_k^s \right)^2 = \frac{1}{2} \sum\limits_{i=1}^{n} \left(y^s - \W_k \bx^s\right)^2.
\]
Our aim is to find
\[
\arg\min\limits_{\W,b} L.
\]
We use
\[
\frac{\partial L}{\partial \W} = \frac{\partial L}{\partial \p}\frac{\partial \p}{\partial \W} = 
\sum\limits_{s=1}^{S} \frac{\partial L}{\partial p^s}\frac{\partial p^s}{\partial \W}=
\]
\[
\sum\limits_{s=1}^{S} \frac{\partial L}{\partial p^s}\bx^s=
\sum\limits_{s=1}^{S}  \left(y^s - \W_k \bx^s \right) (\bx^s)^T
\]

\[
\frac{\partial L}{\partial \p} = \left(p^1 - y^1, \ldots, p^S - y^S \right)  
\]
We will use the following parametrization of the synthetic gradient $\tDL_k = (\alpha_k + 1) \p_k - (\beta_k + 1) \y + \gamma_k \II$. The reason for using this form instead of simply $a_k  \p_k + b_k \y + c_k \II$ is that we are going to show that under DNI this synthetic gradient will converge to the ``real gradient'' $\frac{\partial L}{\partial \p}$, which means showing that $\lim\limits_{k\rightarrow \infty} (\alpha_k, \beta_k, \gamma_k) = (0,0,0)$. Thanks to this choice of parameters $\alpha_k$, $\beta_k$, $\gamma_k$ we have the simple expression for the error
\[
E_k = \left\| \tDL_k - \frac{\partial L}{\partial \p} \right\|_2^2 = 
\]
\[
\left\| (\alpha_k + 1) \p_k - (\beta_k + 1) \y + \gamma_k \II - \right.
\]
\[
\left. \left(p_k^1 - y^1, \ldots, p_k^S - y^S \right)  \right\|_2^2 =
\]
\[
\left\|\left(\alpha_k p_k^1 - \beta_k y^1  + \gamma_k, \ldots, 
\alpha_k p_k^S - \beta_k y^S  + \gamma_k \right)  \right\|_2^2
\]
Parameters $\alpha_k$, $\beta_k$, $\gamma_k$ will be updated using the gradient descent minimizing the error $E$. We have
\[
\frac{\partial E}{\partial \alpha} = \sum\limits_{s=1}^S ( \alpha_k p_k^s - \beta_k y^s  + \gamma_k) p_k^s
\]
\[
\frac{\partial E}{\partial \beta} = - \sum\limits_{s=1}^S ( \alpha_k p_k^s - \beta_k y^s  + \gamma_k) y^s
\]
\[
\frac{\partial E}{\partial \gamma} = \sum\limits_{s=1}^S ( \alpha_k p_k^s - \beta_k y^s  + \gamma_k).
\]
 As prescribed in \citet{DNI}, we start our iterative procedure from the synthetic gradient being equal to zero and we update the parameters by adding the (negative) gradient multiplied by a learning rate $\nu$. This means that we apply the iterative procedure:
\[
\alpha_0 = -1, \;\; \beta_0 = -1, \;\; \gamma_0 = 0  
\]
\begin{equation*} \begin{aligned}
\W_{k+1} =& \W_k - \mu \sum\limits_{s=1}^S \left( (\alpha_k + 1) \p_k^s - \right.\\
& \left. (\beta_k + 1) \y^s +  \gamma_k\right) (\bx^s)^T\\
\alpha_{k+1} =& \alpha_k - \nu \sum\limits_{s=1}^S ( \alpha_k p_k^s - \beta_k y^s  + \gamma_k) p_k^s\\
\beta_{k+1} =& \beta_k + \nu \sum\limits_{s=1}^S ( \alpha_k p_k^s - \beta_k y^s  + \gamma_k) y^s\\
\gamma_{k+1} =& \gamma_k - \nu \sum\limits_{s=1}^S ( \alpha_k p_k^s - \beta_k y^s  + \gamma_k).
\end{aligned} 
\end{equation*} 
Using matrix notation
\begin{equation*} \begin{aligned}
\W_{k+1} &= \W_k - \mu ((\alpha_k + 1) \p_k - (\beta_k + 1) \y +  \gamma_k \II) \bX^T\\
\alpha_{k+1} &= \alpha_k - \nu\left( \alpha_k \| \p_k \|_2^2 - \beta_k \langle \y, \p_k \rangle  + \gamma_k \langle \II, \p_k \rangle \right)\\
\beta_{k+1} &= \beta_k + \nu \left( \alpha_k \langle \p_k, \y \rangle - \beta_k \| \y \|_2^2   + \gamma_k \langle\II,  \y \rangle \right)\\
\gamma_{k+1} &= \gamma_k - \nu \left( \alpha_k \langle \II, \p_k \rangle - \beta_k \langle \II, \y \rangle  + S \gamma_k \right)
\end{aligned} \end{equation*} 
Note, that the subspace given by $\alpha = \beta = \gamma = 0$ is invariant under this mapping. As noted before, this corresponds to the synthetic gradient being equal to the real gradient. Proving the convergence of \SG{} means showing, that a trajectory starting from $\alpha_0 = -1$, $\beta_0 = -1$, $\gamma_0 = 0$ converges to  $\W = \W_0$, $\alpha = \beta = \gamma = 0$, where $\W_0$ are the ``true'' weigts of the linear regression. We are actually going to prove more, we will show that $\W = \W_0$, $\alpha = \beta = \gamma = 0$ is in fact a global attractor, i.e. that any trajectory converges to that point.
Denoting $\omega = (\alpha, \beta, \gamma)^t$ we get
\begin{equation*} \begin{aligned}
\W_{k+1} &= \W_k - \mu ((\alpha_k + 1) \p_k - (\beta_k + 1) \y +  \gamma_k \II) \bX^T\\
\omega_{k+1} &= \omega_k - \nu\left[ \p_k^T | -\y^T | \II^T \right]^T\left[ \p_k^T | -\y^T | \II^T \right] \omega_k
\end{aligned} \end{equation*} 
\begin{equation*} \begin{aligned}
\W_{k+1} &= \W_k - \mu (\p_k - \y) \bX^T - \mu \omega_k^T \left[ \p_k^T | -\y^T | \II^T \right]^T \bX^T\\
\omega_{k+1} &= \omega_k - \nu\left[ \p_k^T | -\y^T | \II^T \right]^T\left[ \p_k^T | -\y^T | \II^T \right] \omega_k.
\end{aligned} \end{equation*} 
Denoting by $\A_k = \left[ \p_k^T | -\y^T | \II^T \right]$ we get
\begin{equation*} \begin{aligned}
\W_{k+1} &= \W_k - \mu (\p_k - \y) \bX^T - \mu \omega^T \A_k^T \bX^T\\
\omega_{k+1} &= \omega_k - \nu \A_k^T \A_k \omega_k.
\end{aligned} \end{equation*}
Multiplying both sides of the first equation by $\bX$ we obtain
\begin{equation*} \begin{aligned}
\W_{k+1} \bX  &= \W_k \bX  - \mu (\p_k - \y) \bX^T \bX - \mu \omega^T \A_k^T \bX^T \bX \\
\omega_{k+1} &= \omega_k - \nu \A_k^T \A_k \omega_k.
\end{aligned} \end{equation*} 
%
Denote $\B = \bX^T \bX$. We get
\begin{eqnarray*}
\p_{k+1} &=& \p_k - \mu \p_k \BB + \mu \y \BB - \mu  \omega_k^T \A_k^T \BB\\
\omega_{k+1} &=& \omega_k - \nu \A_k^T \A_k \omega_k.
\end{eqnarray*}
Denoting $\ee_k = (\y - \p_k)^T $ we get
\begin{eqnarray*}
\ee_{k+1} &=& \ee_k - \mu  \BB \ee_k + \mu  \BB \A_k \omega_k \\
\omega_{k+1} &=& \omega_k - \nu \A_k^T \A_k \omega_k.
\end{eqnarray*}
We will use the symbol $\xi = \A_k \omega_k$. Then 
\begin{equation}
\label{eqn::main_dynamical_system}
\begin{array}{lll}
\ee_{k+1} &=& \ee_k - \mu \BB \ee_k + \mu \BB \xi_k\\
\xi_{k+1} &=& \xi_k - \nu \A_k \A_k^T \xi_k.
\end{array}
\end{equation}
Every vector $v$ can be uniquely expressed as a sum $v = v^{\bot} + v^{\parallel}$ with $\bX v^{\bot} = \mathbf{0}$ and $v^{\parallel} = \bX^T \theta$ for some $\theta$ ($v^{\parallel}$ is a projection of $v$ onto the linear subspace spanned by the columns of $\bX$). Applying this decomposition to $\ee_k = \ee_k^{\bot} + \ee_k^{\parallel}$ we get
\begin{eqnarray*}
\ee_{k+1}^{\bot} &=& \ee_k^{\bot} - \mu (\B \ee_k)^{\bot} + \mu (\B \xi_k)^{\bot} \\
\ekp{k+1} &=& \ekp{k} - \mu (\B \ee_k)^{\parallel}  + \mu (\B \xi_k)^{\parallel}\\
\xi_{k+1} &=& \xi_k - \nu \A_k \A_k^T \xi_k.
\end{eqnarray*}
Note now, that as $\B = \bX^T \bX$, for any vector $v$ there is $(\B v)^{\bot} = \mathbf{0}$, and $(\B v)^{\parallel} = \B v$  (because the operator $v \mapsto v^{\parallel}$ is a projection). Moreover,  $\B v = \B v^{\parallel}$. Therefore
\begin{eqnarray*}
\ee_{k+1}^{\bot} &=& \ee_k^{\bot}\\
\ekp{k+1} &=& \ekp{k} - \mu (\B \ekp{k})  + \mu (\B \xi_k)^{\parallel}\\
\xi_{k+1} &=& \xi_k - \nu \A_k \A_k^T \xi_k.
\end{eqnarray*}
The value $\ee_k^{\bot}$ does not change. Thus, we will be omitting the first equation. Note, that $\ee_k^{\bot}$ is  ``the residue'', the smallest error that can be obtained by a linear regression.\\
For the sake of visual appeal we will denote $\f = \ekp{k}$
\begin{eqnarray*}
\f_{k+1} &=& \f_{k} - \mu \B \f_{k}  + \mu \B \xi_k\\
\xi_{k+1} &=& \xi_k - \nu \A_k \A_k^T \xi_k.
\end{eqnarray*}
Taking norms and using $\| u + v \| \le \|u\| + \|v\|$ we obtain
\begin{eqnarray*}
\|\f_{k+1}\|_2 &\le& \| \f_{k} - \mu \B \f_{k} \|_2  +  \mu \| \B \xi_{k} \|_2\\
\| \xi_{k+1} \|_2^2 &=& \| \xi_k \|_2^2 - 2 \nu \| \A_k^T \xi_k \|_2^2 + \nu^2 \| \A_k \A_k^T \xi_k \|_2^2.
\end{eqnarray*}
Observe that $\| \f_{k} - \mu \B \f_{k} \|_2^2 = \| \f_{k} \|_2^2 - 2 \mu \f_{k} \B \f_k + \mu^2 \| \B \f_k \|^2_2$. As $\B$ is a constant matrix, there exists a constant $b>0$ such that $ v^T \B v \ge b \|v\|_2^2$ for any $v$ satisfying $v^{\parallel} = v$. Therefore $\| \f_{k} - \mu \B \f_{k} \|_2^2 \le \| \f_{k} \|_2^2 - 2 \mu b \| \f_k \|_2^2 + \mu^2 \| \B \|^2 \| \f_k \|_2^2$. Using that and $\| \B \xi_{k} \|_2 \le \| \B\| \| \xi_{k} \|_2$ we get
\begin{eqnarray*}
\|\f_{k+1}\|_2 &\le& \sqrt{1 - 2 \mu b + \mu^2 \| \B \|^2 } \|\f_{k}\|_2  +  \mu \| \B \| \| \xi_{k} \|_2\\
\| \xi_{k+1} \|_2^2 &=& \| \xi_k \|_2^2 - 2 \nu \| \A_k^T \xi_k \|_2^2 + \nu^2 \| \A_k \A_k^T \xi_k \|_2^2.
\end{eqnarray*}
Let us assume that $\A_k \A_k^T \xi_k \neq 0$. In that case the right-hand side of the second equation is a quadratic function is $\nu$, whose minimum value is attained for $\nu =  \frac{\| \A_k^T \xi_k \|_2^2 }{\| \A_k \A_k^T \xi_k \|_2^2 }$. For so-chosen $\nu$ we have
\begin{eqnarray*}
\|\f_{k+1}\|_2 &\le& \sqrt{1 - 2 \mu b + \mu^2 \| \B \|^2 } \|f_{k}\|_2  +  \mu \| \B \| \| \xi_{k} \|_2\\
\| \xi_{k+1} \|_2^2 &=&  \left(1 -
\frac{ \| \A_k^T \xi_{k} \|_2^2}{ \| \A_k \A_k^T \xi_{k} \|_2^2} 
\frac{ \| \A_k^T \xi_{k} \|_2^2}{ \| \xi_{k} \|_2^2}\right) \| \xi_k \|_2^2.
\end{eqnarray*}
Consider a space $\{\f\} \oplus \{\xi\}$ (concatenation of vectors) with a norm $\| \{\f\} \oplus \{\xi\} \|_{\oplus} = \|\f\|_2 + \|\xi\|_2$. 
\[
\|\{\f_{k+1}\} \oplus \{\xi_{k+1}\}\|_{\oplus} \le
\]
\[
\sqrt{1 - 2 \mu b + \mu^2 \| \B \|^2 } \|f_{k}\|_2  +  \mu \| \B \| \| \xi_{k} \|_2 + \\
\]
\[\sqrt{ 1 -
\frac{ \| \A_k^T \xi_{k} \|_2^2}{\| \A_k \A_k^T \xi_{k} \|_2^2} 
\frac{\| \A_k^T \xi_{k} \|_2^2}{ \| \xi_{k} \|_2^2} }  \| \xi_k \|_2 \le \\
\]
Using $\sqrt{1-h} \le 1-\frac{1}{2}h$ we get
\[
\|\{\f_{k+1}\} \oplus \{\xi_{k+1}\}\|_{\oplus} \le \sqrt{1 - 2 \mu b + \mu^2 \| \B \|^2 } \|f_{k}\|_2  +
\]
\[
\left( 1 -
\frac{ \| \A_k^T \xi_{k} \|_2^2}{2 \| \A_k \A_k^T \xi_{k} \|_2^2} 
\frac{\| \A_k^T \xi_{k} \|_2^2}{ \| \xi_{k} \|_2^2} +  \mu \right) \| \xi_k \|_2\\
\]
Note, that $\sqrt{1 - 2 \mu b + \mu^2 \| \B \|^2 } < 1$ for $0 < \mu \le \frac{b}{\|\B\|^2}$. Thus, for
\[
\mu < \min\left\{\frac{b}{\|\B\|^2}, 1 -
\frac{ \| \A_k^T \xi_{k} \|_2^2}{2 \| \A_k \A_k^T \xi_{k} \|_2^2} 
\frac{\| \A_k^T \xi_{k} \|_2^2}{ \| \xi_{k} \|_2^2} \right\},
\]
for every pair $\{\f_{k+1}\} \oplus \{\xi_{k+1}\} \neq \{0\} \oplus \{0\}$ (and if they are zeros then we already converged) there is
\[
\|\{\f_{k+1}\} \oplus \{\xi_{k+1}\}\|_{\oplus} < \|\{\f_{k}\} \oplus \{\xi_{k}\}\|_{\oplus}.
\]
Therefore, by Theorem~\ref{thm::tw_ze_szlenka}, the error pair $\{\f_{k+1}\} \oplus \{\xi_{k+1}\}$ has to converge to $\mathbf{0}$, which ends the proof in the case $\A_k \A_k^T \xi_k \neq 0$. It remains to investigate what happens if $\A_k \A_k^T \xi_k = 0$.

We start by observing that either $\xi_k = 0$ or $\A_k^T \xi_k \neq 0$ and $\A_k \A_k^T \xi_k \neq 0$. This follows directly from the definition $\xi_k = \A_k \omega_k$. Indeed, if $\xi_k \neq 0$ there is $0 < \| \A_k \omega_k \|_2^2 = \omega_k^T \A_k^T \xi_k$ and analogously $0 < \| \A_k^T \xi_k \| = \xi_k^T \A_k \A_k^T \xi_k$.

In case $\xi_k = 0$ there is  $\|\{\f_{k+1}\} \oplus \{\xi_{k+1}\}\|_{\oplus} =  \|\ f_{k+1} \|_2 < 
\sqrt{1 - 2 \mu b + \mu^2 \| \B \|^2 } \| f_{k} \|_2 =$ $ \sqrt{1 - 2 \mu b + \mu^2 \| \B \|^2 } \|\{\f_{k}\} \oplus \{\xi_{k}\}\|_{\oplus} $
and the theorem follows. 
\end{proof}

\begin{thm}
\label{thm::tw_ze_szlenka}
Let $B$ be a finite-dimensional Banach space. Let $f:B \rightarrow B$ be a continuous map such that for every $\x \in B$ there is $\|f(x)\| < \|x\|$. Then for every $x$ there is $\lim\limits_{n \rightarrow \infty} f^n(x) = 0$.
\end{thm}

\begin{proof}
Let $\omega(x) = \{y : \exists_{i_1 < i_2 < \ldots} \lim\limits_{n \rightarrow \infty} f^{i_n}(x) = y \}$. Because $\|f(x)\| < \|x\|$, the sequence $x, f(x), f^2(x), \ldots$ is contained in a ball of a radius $\|x\|$, which due to a finite dimensionality of $B$ is a compact set. Thus, $\omega(x)$ is nonempty. Moreover, from the definition, $\omega(x)$ is a closed set, and therefore it is a compact set. Let $y_0 = \inf_{y \in \omega(x)} \|y\|$ -- which we know exists, due to the compactness of $\omega(x)$ and the continuity of $\|\cdot\|$ (Weierstra{\ss} theorem). But for every $y \in \omega(x)$ there is $f(y) \in \omega(x)$, thus there must be $y_0 = 0$. By definition, for every $\varepsilon$, there exists $n_0$ such that $\|f^{n_0}(x)\| < \varepsilon$. Therefore, for $n > n_0$ $\|f^n(x)\| < \varepsilon$. Therefore, $f^n(x)$ must converge to $0$.
\end{proof}

\textbf{Proposition 2.} 
\emph{
Let us assume that a \SG{} module is trained in each iteration in such a way that it $\epsilon$-tracks true gradient, i.e.  that $\|\mathrm{\SG{}}(h, y) - \partial L/\partial h\| \leq \epsilon$. If $\|\partial h / \partial \theta_{<h}\|$ is upper bounded by some $K$ and there exists a constant $\delta \in (0,1)$ such that in every iteration $\epsilon K \leq \|\partial L/\partial \theta_{<h}\| \tfrac{1-\delta}{1+\delta}$, then the whole training process converges to the solution of the original problem.
}
\begin{proof}
Directly from construction we get that $\|\partial L / \partial \theta_{<h} - \hat \partial L/\hat \partial \theta_{<h}\| = \|(\partial L / \partial h - \mathrm{\SG{}}(h, y) ) \partial h / \partial \theta_{<h}\| \leq \epsilon K$ thus in each iteration there exists such a vector $e$, that $\|e\| \leq \epsilon K$ and $\hat \partial L/\hat \partial \theta_{<h}=\partial L / \partial \theta_{<h} + e$. Consequently, we get a model trained with noisy gradients, where the noise of the gradient is bounded in norm by $\epsilon K$ so, directly from assumptions, it is also upper bounded by $\|\partial L/\partial \theta_{<h}\| \tfrac{1-\delta}{1+\delta}$ and we we get that the direction followed is sufficient for convergence as this means that cosine between true gradient and synthetic gradient is uniformly bounded away (by $\delta$) from zero~\cite{zoutendijk1970nonlinear, gratton2011much}. At the same time, due to Proposition 1, we know that the assumptions do not form an empty set as the \SG{} module can stay in an $\epsilon$ neighborhood of the gradient, and both norm of the synthetic gradient and $\|\partial h / \partial \theta_{<h}\|$ can go to zero around the true critical point.
\end{proof}

\textbf{Corollary 1.} 
\emph{
For a deep linear model and an MSE objective, trained with a linear \SG{} module attached between two of its hidden layers, there exist learning rates in each iteration such that it converges to the critical point of the original problem.
}
\begin{proof}
Denote the learning rate of the main model by $\mu$ and learning rate of the \SG{} module by $\nu>0$ and put $\mu = \epsilon  \max(0, \|e\| - 1/(3 \|\partial h / \partial \theta_{<h}\|) \| \partial L / \partial \theta_{<h} \|)$, where $\epsilon$ is a small learning rate (for example found using line search) and $e$ is the error \SG{} will make in the next iteration. The constant $1/3$ appears here as it is equal to $(1-\delta)/(1+\delta)$ for $\delta=0.5$ which is a constant from Proposition 2, which we will need later on. Norm of $e$ consists of the error fitting term $L_{\mathrm{\SG{}}}$ which we know, and the term depending on the previous $\mu$ value, since this is how much the solution for the \SG{} problem evolved over last iteration. In such a setting, the main model changes iff 
\begin{equation}
\label{eq:ineqdeep}
\|e\| \| \partial h / \partial \theta_{<h}\| < 1/3 \| \partial L / \partial \theta_{<h}\|.
\end{equation}
First of all, this takes place as long as $\nu$ is small enough since the linear \SG{} is enough to represent $\partial L / \partial h$ with arbitrary precision (Proposition 1) and it is trained to do so in a way that always converges (as it is a linear regression fitted to a linear function). So in the worst case scenario for a few first iterations we choose very small $\mu$ (it always exists since in the worst case scenario $\mu=0$ agrees with the inequality).  Furthermore, once this happens we follow true gradient on $\theta_{>h}$ and a noisy gradient on $\theta_{<h}$. Since the noise is equal to $e \partial h / \partial \theta_{<h}$ we get that
$$
\|e \partial h / \partial \theta_{<h}\| \leq \|e\|\|\partial h / \partial \theta_{<h}\| <  1/3 \| \partial L / \partial \theta_{<h}\|,
$$
 which is equivalent to error for $\theta_{<h}$ being upper bounded by $(1-\delta)/(1+\delta) \| \partial L / \partial h\|$ for $\delta=0.5$ which matches assumptions of Proposition 2, thus leading to the convergence of the model considered. If at any moment we lose track of the gradient again -- the same mechanism kicks in - $\mu$ goes down for as long as the inequality~\eqref{eq:ineqdeep} does not hold again (and it has to at some point, given $\nu$ is positive and small enough).
\end{proof}

\begin{figure*}[htb]
\includegraphics[width=\textwidth]{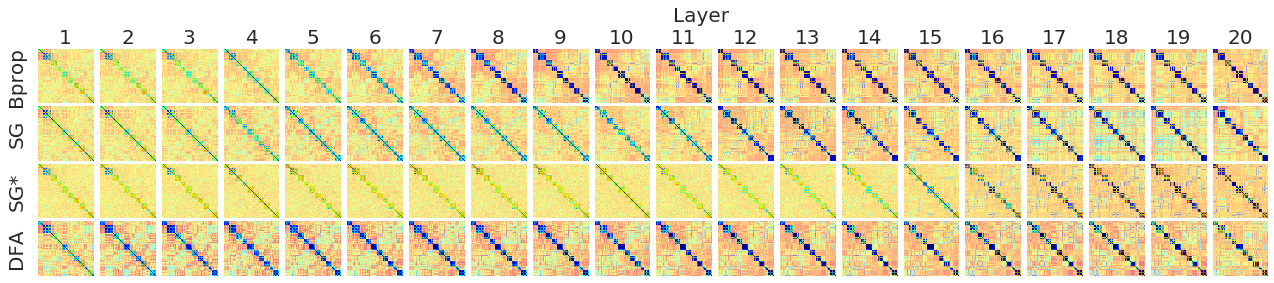}
\includegraphics[width=\textwidth]{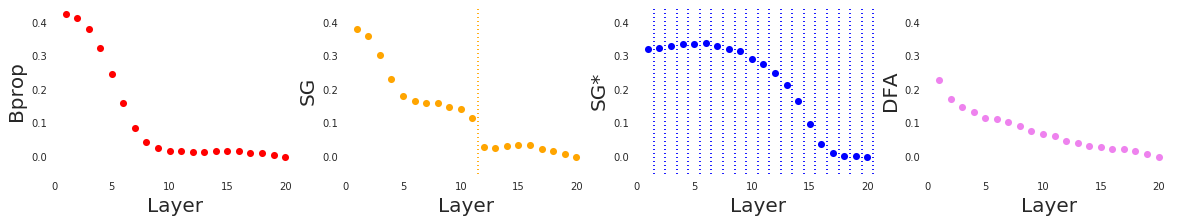}
\caption{(top) Representation Dissimilarity Matrices for a label ordered sample from MNIST dataset pushed through 20-hidden layer deep relu networks trained with backpropagation (top row), a single \SG{} attached between layers 11 and 12 (2nd row), \SG{} between every pair of layers (3rd row), and the DFA model (4th row). Notice the moment of appearance of dark blue squares on a diagonal in each learning method, which shows when a clear inner-class representation has been learned. For visual confidence off block diagonal elements are semi transparent. (bottom) L2 distance between diagonal elements at a given layer and the same elements at layer 20. Dotted lines show where \SG{}s are inserted. With a single \SG{} module we can see that there is the representation is qualitatively different for the first part of the network (up to layer 11) and the rest. For fully unlocked model the representation constantly evolves through all the layers, as opposed to backprop which has a nearly constant representation correlation from layer 9 forward. Also due to DFA mathematical formulation it tries to solve the task as early as possible thus leading to nearly non-evolving representation correlation after the very first layer.}
\label{fig:dfardm}
\end{figure*}

\section{Technical details}
\label{SM:experiments}

All experiments were performed using TensorFlow~\cite{abadi2016tensorflow}. In all the experiments \SG{} loss is the MSE between synthetic and true gradients. Since all \SG{}s considered were linear, weights were initialized to zeros so initially \SG{} produces zero gradients, and it does not affect convergence (since linear regression is convex).

\subsection*{Datasets}

Each of the artificial datasets is a classification problem, consisting of $\X$ sampled from $k$-dimensional Gaussian distribution with zero mean and unit standard deviation. For $k=2$ we sample 100 points and for $k=100$ we sample 1000. Labels $\y$ are generated in a way depending on the dataset name:
\begin{itemize}
\item linear$k$ - we randomly sample an origin-crossing hyperplane (by sampling its parameters from standard Gaussians) and label points accordingly,
\item noisy$k$ - we label points according to linear$k$ and then randomly swap labels of 10\% of samples,
\item random$k$ - points are labeled completely randomly.
\end{itemize}

We used one-hot encoding of binary labels to retain compatibility with softmax-based models, which is consistent with the rest of experiments. However we also tested the same things with a single output neuron and regular sigmoid-based network and obtained analogous results.

\subsection*{Optimisation}

Optimisation is performed using the Adam optimiser~\cite{kingma2014adam} with a learning rate of $3e-5$. This applies to both main model and to \SG{} module.

\subsection*{Artificial datasets}

\begin{table}[htb]
\centering
\begin{tabular}{lllrr}
\toprule
  dataset  & model & MSE & log loss \\
\midrule
  linear2  & shallow & 0.00000 & \textbf{0.03842} \\
  linear100  & shallow & 0.00002 & \textbf{0.08554}\\
  noisy2  & shallow & 0.00000 & \textbf{0.00036} \\
  noisy100  & shallow & 0.00002 & \textbf{0.00442}\\
  random2  & shallow & 0.00000 & 0.00000 \\
  random100  & shallow & 0.00004 & 0.00003  \\

  noisy2  & deep & 0.00000 & 0.00000 \\
  noisy100  & deep & 0.00001 & \textbf{0.00293}  \\
  random2  & deep & 0.00000 & 0.00000 \\
  random100 & deep &  0.00001 & 0.00004 \\
\bottomrule  
\end{tabular}
\caption{Differences in final losses obtained for various models/datasets when trained with \SG{} as compared to  model trained with backpropagation. Bolded entries denote experiments which converged to a different solution. \emph{linear$k$} is $k$ dimensional, linearly separable dataset, \emph{noisy} is linearly separable up to 10\% label noise, and \emph{random} has completely random labeling. Shallow models means linear ones, while deep means 10 hidden layer deep linear models. Reported differences are averaged across 10 different datasets from the same distributions.}
\label{tab:diffs}
\end{table}
Table~\ref{tab:diffs} shows results for training linear regression (shallow MSE), 10 hidden layer deep linear regression (deep MSE), logistic regression (shallow log loss) and 10 hidden layer deep linear classifier (deep log loss). Since all these problems (after proper initialisation) converge to the global optima, we report the difference between final loss obtained for \SG{} enriched models and the true global optimum.

\subsection*{MNIST experiments}

Networks used are simple feed forward networks with $h$ layers of 512 hidden relu units followed by batch normalisation layers. The final layer is a regular 10-class softmax layer. Inputs were scaled to $[0,1]$ interval, besides that there was no preprocessing applied.

\subsection*{Representational Dissimilarity Matrices}

In order to build RSMs for a layer $h$ we sample 400 points (sorted according to their label) from the MNIST dataset, $\{x_i\}_{i=1}^{400}$ and record activations on each of these points, $h_i = h(x_i)$. Then we compute a matrix $\mathrm{RSM}$ such that $\mathrm{RSM}_{ij} = 1-\mathrm{corr}(h_i, h_j)$. Consequently a perfect RSM is a block diagonal matrix, thus elements of the same class have a representation with high correlation and the representations of points from two distinct classes are not correlated.
Figure~\ref{fig:dfardm} is the extended version of the analogous Figure~\ref{fig:rdms} from the main paper where we show RDMs for backpropagation, a single \SG{}, \SG{} in-between every two layers, and also the DFA model, when training 20 hidden layer deep relu network.

\subsection*{Linear classifier/regression probes}
One way of checking the degree to which the actual classification problem is solved at every layer of a feedforward network is
to attach linear classifiers to every  hidden layer and train them on the main task without backpropagating through the rest of the network. This way we can make a plot of train accuracy obtained from the representation at each layer. As seen in Figure~\ref{fig:probes} (left) there is not much of the difference between such analysis for backpropagation and a single \SG{} module, confirming our claim in the paper that despite different representations in both sections of \SG{} based module - they are both good enough to solve the main problem. We can also that DFA tries to solve the classification problem bottom-up as opposed to up-bottom -- notice that for DFA we can have 100\% accuracy after the very first hidden layer, which is not true even for backpropagation.
\begin{figure}[htb]
\includegraphics[width=0.23\textwidth]{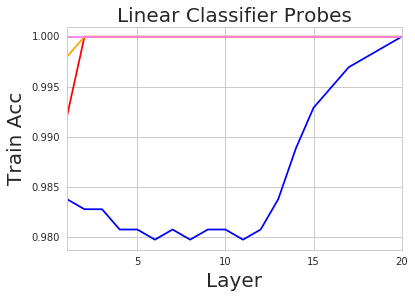}
\includegraphics[width=0.23\textwidth]{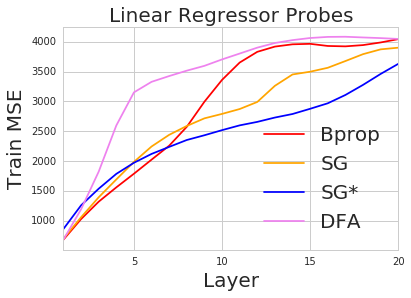}
\caption{\emph{Left:} Training accuracy at each linear classifier probe. \emph{Right:} MSE for each linear regressor probe.}
\label{fig:probes}
\end{figure}

We also introduced a new kind of linear probe, which tries to capture how much computation (non-linear transformations) are being used in each layer. To achieve this, we attach a linear regressor module after each hidden layer and regress it (with MSE) to the input of the network. This is obviously label agnostic approach, but measures how non-linear the transformations are up to the given hidden layer. Figure~\ref{fig:probes} (right) again confirms that with a single \SG{} we have two parts of the network (thus results are similar to RDM experiments) which do have slightly different behaviour, and again show clearly that DFA performs lots of non-linear transformations very early on compared to all other methods.
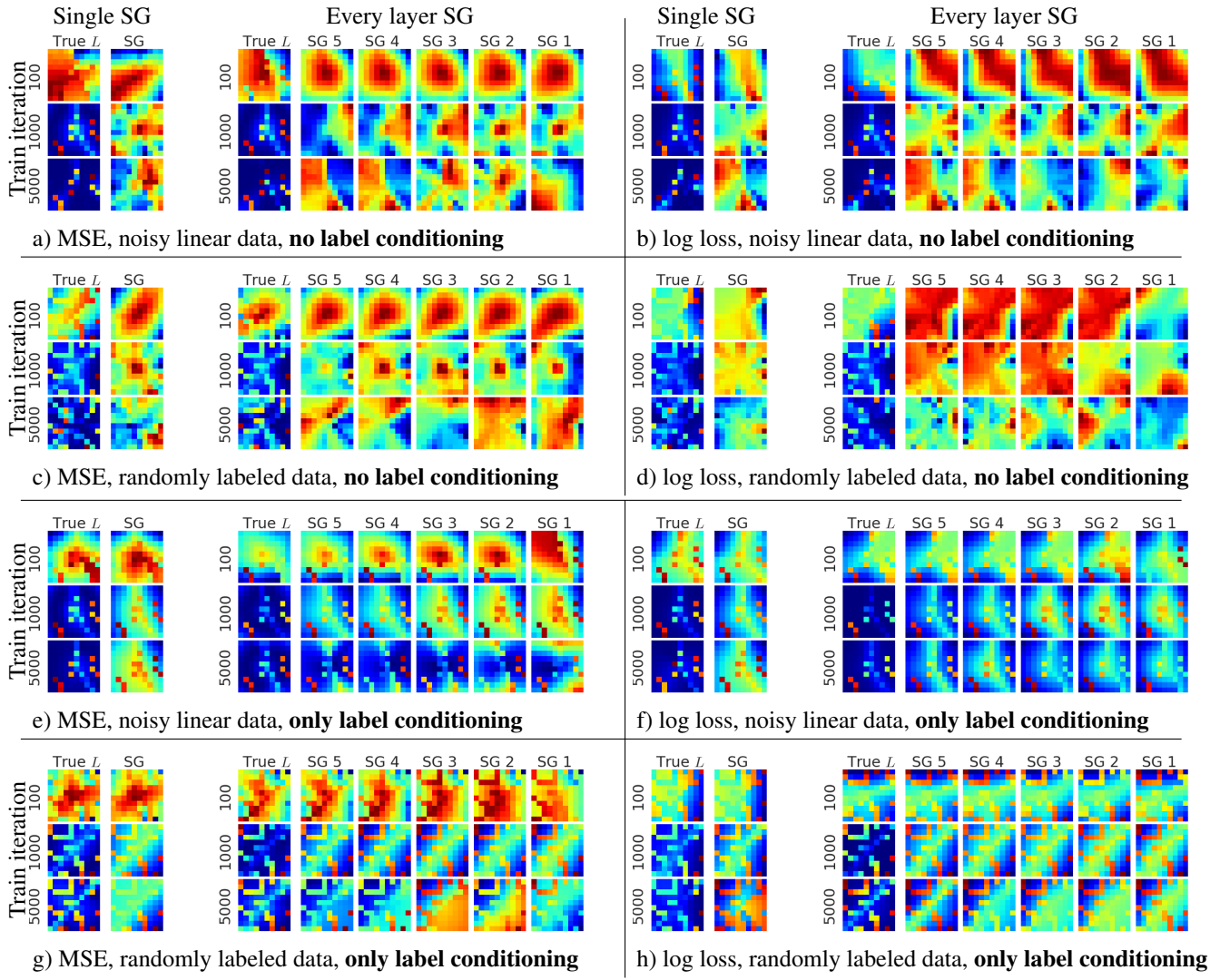
\begin{figure*}[h]
\begin{tikzpicture}
    \node[anchor=north west] at (0,3.5) {
        \includegraphics[height=3.1cm]{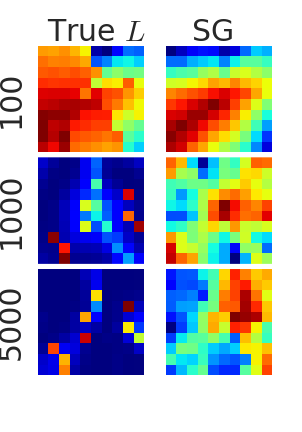}
        \includegraphics[height=3.1cm]{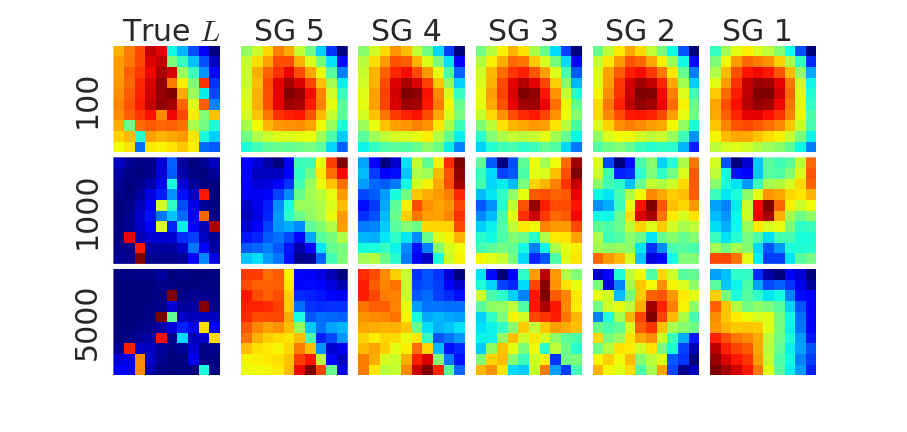}
        \includegraphics[height=3.1cm]{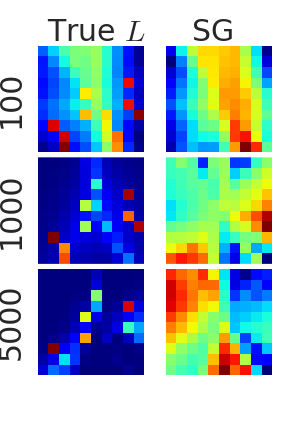}
        \includegraphics[height=3.1cm]{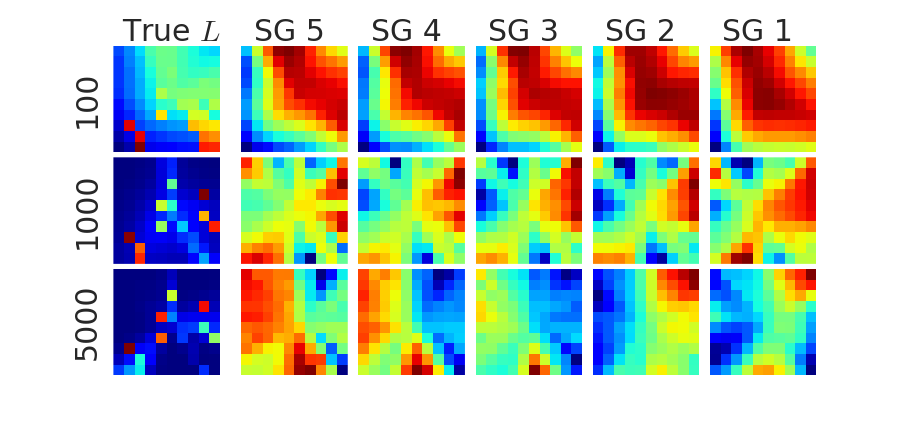}
    };
    \node[anchor=north west] at (0,0.0) {
        \includegraphics[height=3.1cm]{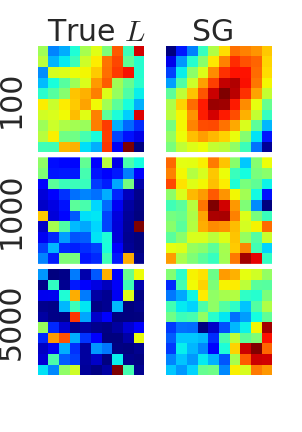}
        \includegraphics[height=3.1cm]{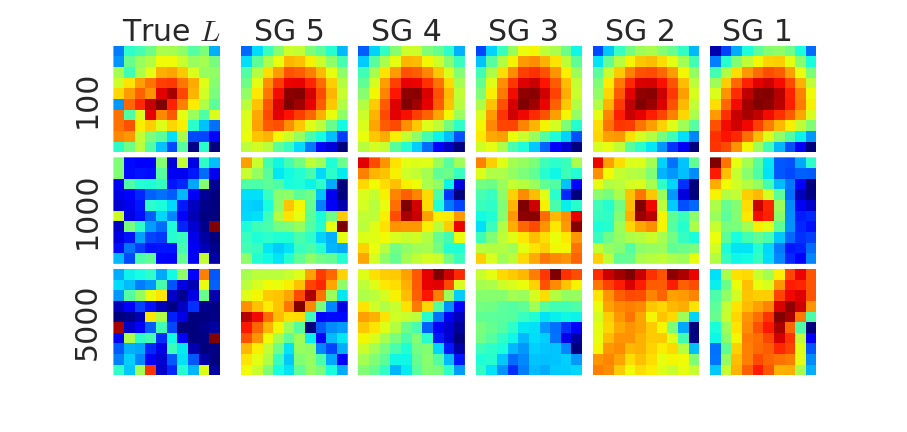}
        \includegraphics[height=3.1cm]{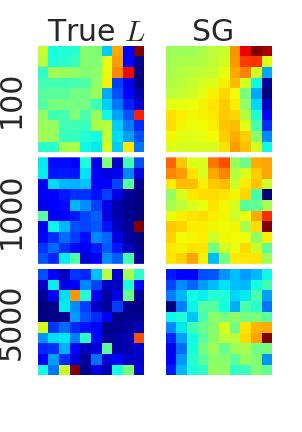}
        \includegraphics[height=3.1cm]{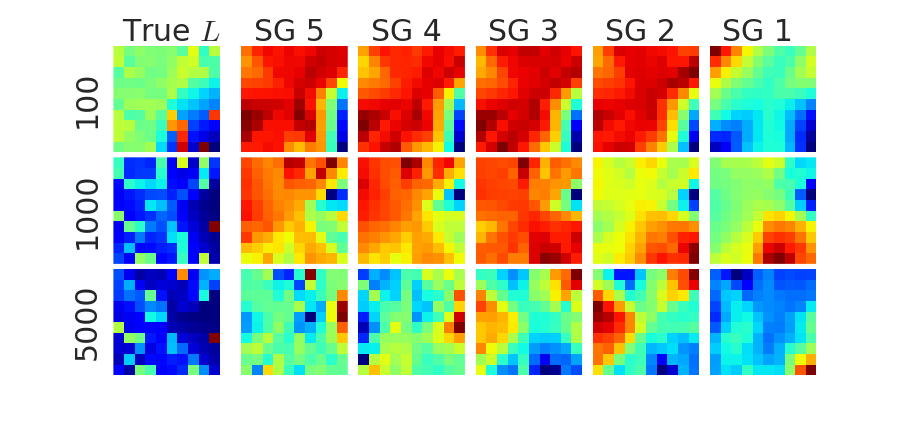}
    };
    \draw (8.85, -3.5) -- (8.85, 3.5);
    \draw (0, 0) -- (17, 0);

    \node[anchor=west, align=left, rotate=90] at (-0.05, 0.75) {Train iteration};
    
    \node[anchor=west, align=left] at (0.35,3.5) {Single SG};
    \node[anchor=west, align=left] at (4.35,3.5) {Every layer SG};

    \node[anchor=west, align=left] at (8.9 + 0.3,3.5) {Single SG};
    \node[anchor=west, align=left] at (8.9 + 4.3,3.5) {Every layer SG};
    \node[anchor=west, align=left, rotate=90] at (-0.05, -2.75) {Train iteration};
    
    \node[anchor=west, align=left] at (0.05,0.25) {a) MSE, noisy linear data, \textbf{no label conditioning}};

    \node[anchor=west, align=left] at (8.9,0.25) {b) log loss, noisy linear data, \textbf{no label conditioning}};

    \node[anchor=west, align=left] at (0.05,-3.25) {c) MSE, randomly labeled data, \textbf{no label conditioning}};

    \node[anchor=west, align=left] at (8.9,-3.25) {d) log loss, randomly labeled data, \textbf{no label conditioning}};
\end{tikzpicture}

\begin{tikzpicture}
    \node[anchor=north west] at (0,3.5) {
        \includegraphics[height=3.1cm]{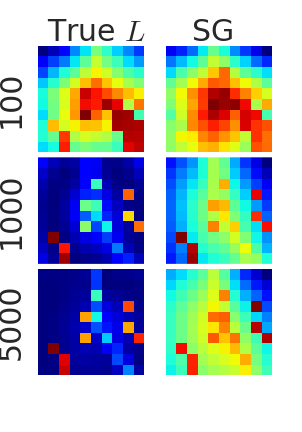}
        \includegraphics[height=3.1cm]{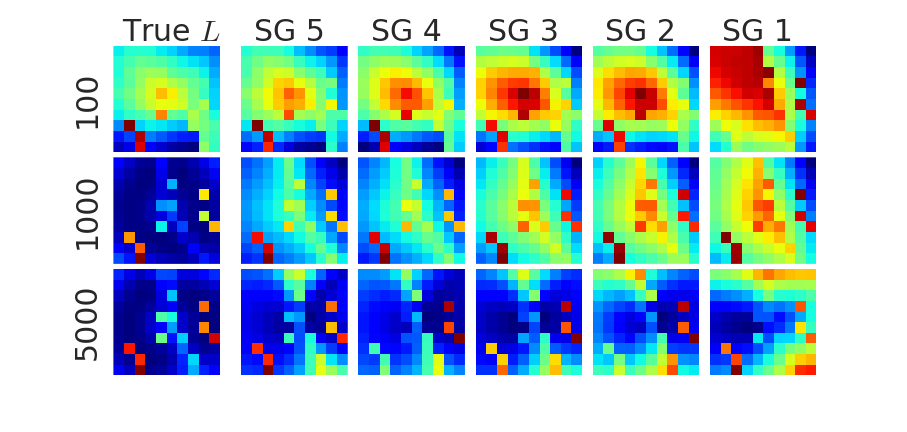}
        \includegraphics[height=3.1cm]{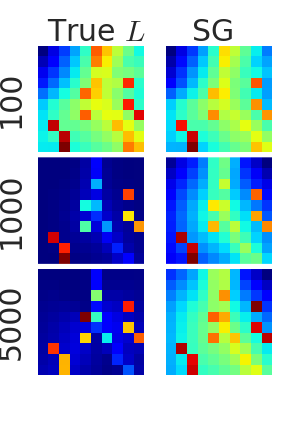}
        \includegraphics[height=3.1cm]{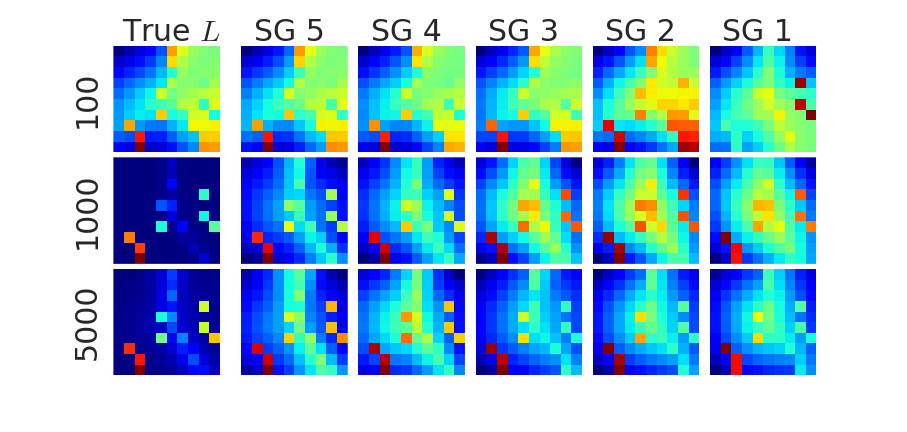}
    };
    \node[anchor=north west] at (0,0) {
        \includegraphics[height=3.1cm]{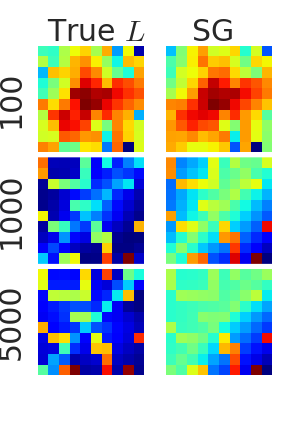}
        \includegraphics[height=3.1cm]{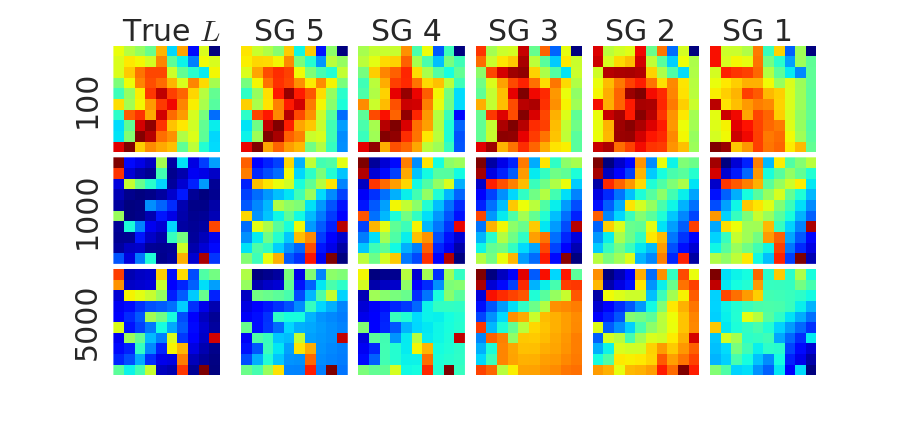}
        \includegraphics[height=3.1cm]{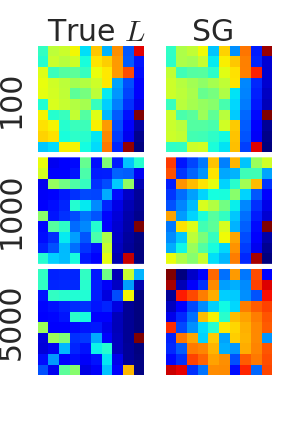}
        \includegraphics[height=3.1cm]{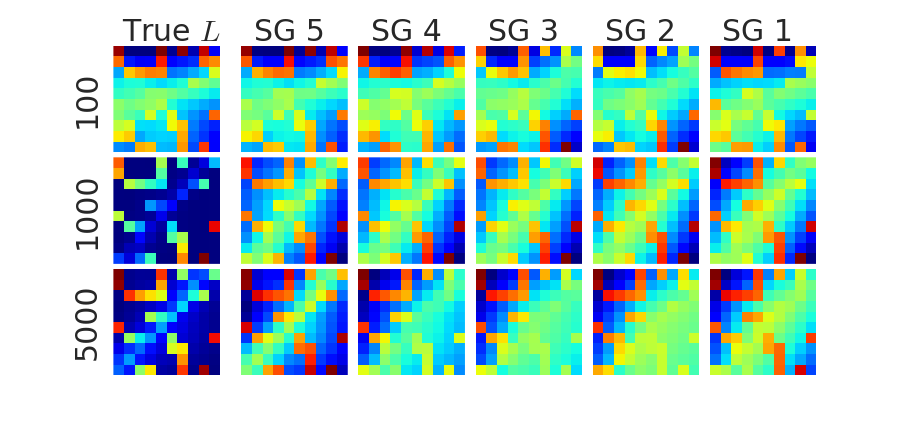}
    };
    \draw (8.85, -3.5) -- (8.85, 3.5);
    \draw (0, 0) -- (17, 0);
    \draw (0, 3.5) -- (17, 3.5);
 
    \node[anchor=west, align=left, rotate=90] at (-0.05, 0.75) {Train iteration};
    

    \node[anchor=west, align=left, rotate=90] at (-0.05, -2.75) {Train iteration};
    
    \node[anchor=west, align=left] at (0.05,0.25) {e) MSE, noisy linear data, \textbf{only label conditioning}};

    \node[anchor=west, align=left] at (8.9,0.25) {f) log loss, noisy linear data, \textbf{only label conditioning}};

    \node[anchor=west, align=left] at (0.05,-3.25) {g) MSE, randomly labeled data, \textbf{only label conditioning}};

    \node[anchor=west, align=left] at (8.9,-3.25) {h) log loss, randomly labeled data, \textbf{only label conditioning}};
\end{tikzpicture}
\caption{Visualisation of the true loss and the loss extracted from the \SG{} module. In each block left plot shows an experiment with a single \SG{} attached and the right one with a \SG{} after each hidden layer. Note, that in this experiment the final loss is actually big, thus even though the loss reassembles some part of the noise surface, the bright artifact lines are actually keeping it away from the true solution.}
\label{fig:onlyy}
\label{fig:onlyh}
\end{figure*}

\subsection*{Loss estimation}
In the main paper we show how \SG{} modules using both activations and labels are able to implicitly describe the loss surface reasonably well for most of the training, with different datasets and losses. For completeness, we also include the same experiment for \SG{} modules which do not use label information (Figure~\ref{fig:onlyh} (a) - (d)) as well as a module which does not use activations at all\footnote{This is more similar to a per-label stale gradient model.} (Figure~\ref{fig:onlyy} (e) - (h))). There are two important observations here: Firstly, none of these two approaches provide a loss estimation fidelity comparable with the full \SG{} (conditioned on both activations and labels). This gives another empirical confirmation for correct conditioning of the module. Secondly, models which used only labels did not converge to a good solutions after 100k iterations, while without the label \SG{} was able to do so (however it took much longer and was far noisier). 

\end{appendices}

\end{document}